\documentclass[twoside]{article}

\usepackage[accepted]{aistats2021}

\usepackage{xcolor}
\definecolor{linkcolor}{RGB}{0,128,255}
\usepackage[colorlinks=true,allcolors=linkcolor,pageanchor=true,plainpages=false,pdfpagelabels,bookmarks,bookmarksnumbered]{hyperref}

\usepackage{AVTcommands}
\usepackage{AVTenvironments}
\usepackage[authoryear]{AVTcitations}
\usepackage{microtype}
\usepackage[bottom]{footmisc}

\usepackage{subfigure}
\usepackage{xfrac}
\usepackage{algorithm}
\usepackage{algorithmic}
\usepackage{amsthm}
\usepackage{dblfloatfix}

\newtheorem{theorem}{Theorem}
\newtheorem{proposition}[theorem]{Proposition}
\newtheorem{result}[theorem]{Result}

\newtheorem{definition}[theorem]{Definition}

\bibliography{references.bib}

\let\UrlFont\scshape

\runningauthor{Samuel Cohen, Giulia Luise, Alexander Terenin, Brandon Amos, Marc Peter Deisenroth}

\begin{document}
\twocolumn[

\aistatstitle{Aligning Time Series on Incomparable Spaces}

\aistatsauthor{ Samuel Cohen \And Giulia Luise \And Alexander Terenin}
\aistatsaddress{Centre for Artificial Intelligence\\ University College London \And Department of Computing \\ Imperial College London \And Department of Mathematics\\Imperial College London}

\aistatsauthor{Brandon Amos \And Marc Peter Deisenroth}
\aistatsaddress{Facebook AI Research \And Centre for Artificial Intelligence\\ University College London } ]

\begin{abstract}
Dynamic time warping (DTW) is a useful method for aligning, comparing and combining time series, but it requires them to live in comparable spaces.
In this work, we consider a setting in which time series live on different spaces without a sensible ground metric, causing DTW to become ill-defined.
To alleviate this, we propose Gromov dynamic time warping (GDTW), a distance between time series on potentially incomparable spaces that avoids the comparability requirement by instead considering intra-relational geometry.
We demonstrate its effectiveness at aligning, combining and comparing time series living on incomparable spaces.
We further propose a smoothed version of GDTW as a differentiable loss and assess its properties in a variety of settings, including barycentric averaging, generative modeling and imitation learning.
\end{abstract}

\section{Introduction}

Data is often gathered sequentially in the form of a time series, which consists of a sequence of data points observed at successive time points.
Elements of such sequences are correlated through time, and comparing time series requires one to take the direction of time into account.
To define a meaningful similarity measure between time series, \textcite{saoke78} proposed \emph{dynamic time warping} (DTW), a distance over the space of time series.
DTW consists of a minimal-cost alignment problem and is solved efficiently via dynamic programming.

Dynamic time warping enables one to tackle a large range of temporal problems, including aligning, comparing, and averaging time series.
In particular, DTW can be employed as a loss function within larger learning frameworks: in this setting, \textcite{pmlr-v70-cuturi17a} propose \emph{soft DTW}, which consists of a smoothed DTW objective possessing a differentiable gradient which can result in better behavior when employing gradient-based methods \cite{softdtw_gak}.

\begin{table}[b!]
\footnotesize Code available at: \url{https://github.com/samcohen16/Aligning-Time-Series}.
\vspace*{6.19ex}
\end{table}

DTW and its variants require a sensible cost function to be defined between samples from the two time series.
The specification of such cost functions is often hard, and limits the applicability of DTW.
For example, in cases where the time series are invariant under symmetries, such as sequences of word embeddings which are only identified up to a rotation of latent space, one needs to solve a spatial alignment problem to compare the two sequences sensibly.

\textcite{ctw} propose an extension of DTW that addresses this issue by jointly optimizing spatial and temporal projections that align the time series.  
\textcite{vayer2020time} introduce a similar extension of DTW that consists in making the cost invariant with respect to specific sets of invariances, such as for example rotations.
In these approaches, one still requires the definition of a cost function between samples from the two time series, along with a potentially large pre-defined set of transformations to optimize over. 
On the other hand, in multi-modal settings, one considers time series that live on incomparable spaces: for example, the configuration space of a robotic arm and its representation as pixels of a video frame.
In such cases, defining a sensible distance between different representations and a sensible space of symmetries is impractical, as it would require detailed understanding of the objects we wish to study.

In this work, we propose to tackle the incomparability and invariance problems simultaneously by relaxing our notion of equality in a manner inspired by recent ideas from the optimal transport literature.
Using connections between DTW and the Wasserstein distance \cite{Kantorovich2006}, we propose \emph{Gromov dynamic time warping} (GDTW), which compares two time series by contrasting their intra-relational geometries, analogously to the Gromov--Wasserstein distance of isometry classes of metric-measure spaces \cite{journals/focm/Memoli11}.
This allows one to compare two time series without requiring a similarity notion between their samples.
The resulting procedure automatically incorporates invariances into the distance, without requiring said invariances or symmetry-specific constraints to be manually specified.

\paragraph{Contributions.} (1) We introduce a new distance between time series that is well-defined on incomparable spaces with naturally built-in invariance to isometries, and (2) a smoothed extension with better-behaved gradients. (3) We propose an efficient Frank--Wolfe-inspired  algorithm for computing it, and (4) we apply Gromov DTW as a loss function in a wide range of settings, including barycentric averaging, generative modeling and imitation learning.

\paragraph{Notation.}
Let $(\c{X},d_{\c{X}})$ be a compact metric space, and let a \emph{time series} $\v{x}$ of length $T\in\N$ be an element of $\c{X}^T$.
Let $\c{A}(m,n) \subseteq \{0,1\}^{m\times n}$ be the set of \emph{alignment matrices}, which are binary matrices containing a path of ones from the top-left to the bottom-right corner, allowing only bottom, right or diagonal bottom-right moves.
Given a matrix $\m{A}\in \c{A}(m,n)$ and a 4-dimensional array $\m{L} \in \R^{m \x n \x m \x n}$, define the matrix $(\m{L}\ox\m{A})_{ij} = \del[1]{\sum_{kl}L_{ijkl} A_{kl}}_{ij}$.
Denote the Frobenius matrix inner product by $\innerprod{\cdot}{\cdot}_{\f{F}}$. Define the probability simplex $\Delta_J=\{ q \in \R^J, \,\,q_j\geq 0 \textnormal{ for }j=1,\dots,J,\,\,\sum_j q_j=1\}$. Finally, $\v{x}_{:i}$ corresponds to the first $i$ time steps of $\v{x}$.

\begin{figure*}
    \centering
    
  \includegraphics{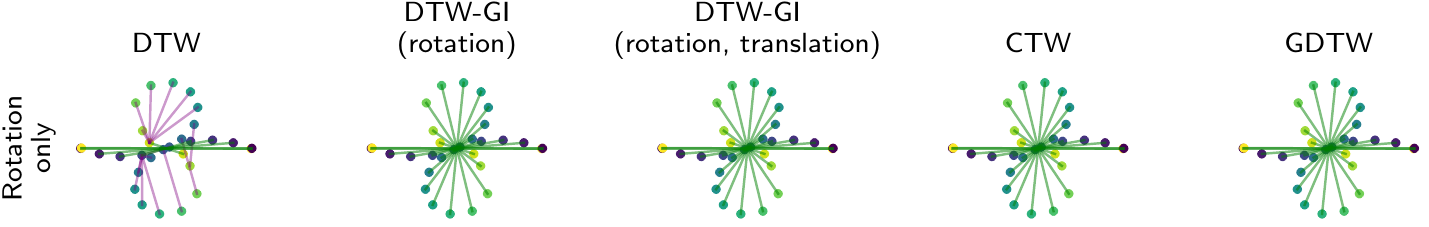}
  \\[0.5\baselineskip]
  \includegraphics{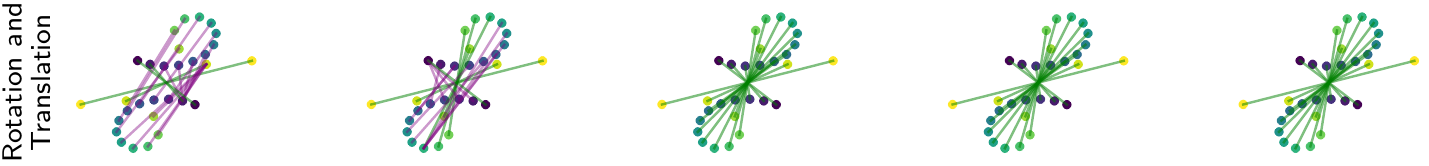}

    \caption{ Alignment of time series equivalent  up to rotation by 180 degrees (top), and  up to rotation and translation (bottom). Node coloring represents time (dark purple: $t=0$, yellow: $t=T$), and edge coloring represents alignment correctness (respectively green and purple for correct and incorrect matchings). The CTW and DTW-GI (rotation, translation) baselines recover the correct alignment. GDTW also recovers the correct alignments, but without needing to manually specify a cost function or symmetries. DTW-GI (rotation) fails in the translational setting, and DTW fails in both.}
    \label{fig:alignment_exp}
\end{figure*}

\section{Dynamic Time Warping for Time Series Alignment}
\label{sec:dtw}

 \textcite{saoke78} consider the problem of aligning two time series $\v{x} \in \c{X}^{T_x}$ and $\v{y} \in \c{X}^{T_y}$, where potentially $T_x \neq T_y$.
This is formalized as
\[
\label{eqn:dtw}
\f{DTW}(\v{x},\v{y}) = \min_{\m{A} \in \c{A}(T_x,T_y)} \innerprod{\m{D}}{\m{A}}_{\f{F}}
\]
where $D_{ij} = d_\c{X}(x_i,y_j)$ is the pairwise distance matrix.
This problem amounts to finding an alignment matrix that minimizes the total alignment cost.
The objective \eqref{eqn:dtw} can be computed in $O(T_xT_y)$ by leveraging the dynamic programming forward recursion
\[
\label{eqn:dynamic-programming}
\begin{aligned}
&\f{DTW}(\v{x}_{:i},\v{y}_{:j}) = d_{\c{X}}(x_i,y_j)
\\
&+\min \del{\f{DTW}_{i-1,j},\f{DTW}_{i-1,j-1}, \f{DTW}_{i,j-1}},
\end{aligned}
\]
where $\f{DTW}_{i,j} = \f{DTW}(\v{x}_{:i},\v{y}_{:j})$.
The optimal alignment matrix $\m{A}^*$ can then be obtained by tracking the optimal path backwards.
DTW is a more flexible choice for comparing time series than element-wise Euclidean distances, because it allows one to compare time series of different sampling frequencies due to its ability to "warp" time.
In particular, two time series can be close in DTW even if $T_x \neq T_y$.
DTW has been used in a number of settings, including time series averaging, clustering \cite{PETITJEAN201276,10.1016/j.patcog.2017.08.012} and feature extraction \cite{655778,10.1007/s10618-015-0418-x}.

A limitation of DTW is the discontinuity of its gradient, which can affect the performance of gradient descent algorithms. To address this, \textcite{pmlr-v70-cuturi17a} introduced a soft version of DTW.  The minimum in~\eqref{eqn:dtw} is replaced with a softened version, yielding
\[
\label{eq:sdtw}
\f{DTW}_\gamma(\v{x},\v{y})=-\gamma \log \sum_{\mathclap{\m{A} \in \c{A}(T_x,T_y)}} \exp\big({-\tfrac{1}{\gamma}\innerprod{\m{D}}{\m{A}}_{\f{F}}}\big)
.
\]
DTW is recovered in the limit $\gamma \-> 0$. They also discuss a softened version of the optimal alignment matrix $\m{A}^*$, given by the softened argmin
\[\label{eq:softargmin}
\!\!\!\operatorname*{{\argmin}^\gamma}_{\m{A}\in \mathcal{A}(T_x,T_y)}\!\innerprod{\m{D}}{\m{A}}_{\f{F}} = C_{\v{x},\v{y}}^{-1} \sum_{\mathclap{\m{A}\in \mathcal{A}(T_x,T_y)}}\exp\big({-\tfrac{1}{\gamma}\innerprod{\m{D}}{\m{A}}_{\f{F}}}\big)\m{A},\!
\!
\]
 where $\gamma\geq 0$ is a smoothing parameter and $C_{\v{x},\v{y}}$ is the normalizing constant of the unnormalized density $P(\m{A}) \propto \smash{\raisebox{-.2ex}{\ensuremath{e^{-\frac{1}{\gamma}\innerprod{\m{D}}{\m{A}}_{\f{F}}}}}}$.
While they consider temporal variability, DTW and soft DTW are not invariant under transformations, such as translations and rotations, which can limit their application to settings where time series are obtained only up to isometric transformations, such as word embeddings.
To alleviate this, \textcite{vayer2020time} propose
\[
\f{DTW-GI}(\v{x},\v{y})=\min_{f\in\c{F}}\f{DTW}(\v{x},f(\v{y})),
\]
which gives a distance between time series that is invariant under a set of transformations $\mathcal{F}$,
where $f$ is applied elementwise to points of the time series; \textcite{vayer2020time} consider orthonormal transformations, such as rotations.
In more general settings, this requires one to optimize over a potentially large space of transformations $\c{F}$, which becomes infeasible if $\v{x}$ and $\v{y}$ are too different.

Similarly to DTW-GI, Canonical Time Warping (CTW) \cite{ctw} consists of aligning the data temporally via DTW and spatially via canonical correlation analysis (CCA). CTW is defined as
\[
    \f{CTW}(\v{x},\v{y}) = \min_{\substack{\m{W}_x, \m{W}_y\\ \m{V}_x, \m{V}_y}} \Vert \m{V}_x\v{x}\m{W}_x +\m{V}_y\v{y}\m{W}_y  \Vert^2_{\f{F}}
\]
with constraints on matrices $\m{W}_x, \m{W}_y, \m{V}_x, \m{V}_y$, which make CTW invariant to translations, rotations and scaling at optimality. Optimization is performed by alternation on $\m{V}_x, \m{V}_y$ via DTW, and on  $\m{W}_x, \m{W}_y$ via CCA. In particular, the former matrices align $\v{x}$ and $\v{y}$ temporally whilst the latter ones align the time series spatially by projecting the temporally-aligned time series onto a common subspace on which they are maximally correlated. \textcite{gctw} generalize CTW to allow for the alignment of multiple time series, and \textcite{dctw} allow for nonlinear projections. \textcite{mandw} leverage manifold learning to align the time series spatially in conjuction with DTW for temporal alignment.

\subsection{Connecting DTW and Optimal Transport}
\label{sec:ot}
Optimal transport \cite{peyre2019computational} allows one to compare and average measures in a way that incorporates the geometry of the underlying space on which they are defined.
Such approaches can be intuitively connected to DTW by observing that time series are essentially discrete measures equipped with an ordering.
This allows one to view the alignment matrices in the DTW objective as analogues of coupling matrices that appear in the Kantorovich formulation of the classical optimal transport problem \cite{alma991005863149705596}. To formalize this, consider the Wasserstein distance between discrete measures.
Let $\mu_x=\sum_{i=1}^m p_i\delta_{x_i}$, $\mu_y=\sum_{i=1}^n q_i \delta_{y_i}$  be discrete probability measures with $\v{p} \in \Delta_{m}, \v{q} \in \Delta_{n}$, and set $D_{ij}=d_{\c{X}}(x_i,x_j)$.
Define the Wasserstein distance between discrete measures $\mu_x$ and $\mu_y$ as
\[
\label{eq:wass}
\f{W}(\mu_x,\mu_y)=\min_{\m{T} \in \Pi(\v{p},\v{q})} \innerprod{\m{D}}{\m{T}}_{\f{F}},
\]
where $\Pi(\v{p},\v{q})$ is the set of coupling matrices with marginals $\v{p}$ and $\v{q}$.
Equation \eqref{eq:wass} clearly resembles \eqref{eqn:dtw}, and in both cases the objective consists of the minimization of the element-wise dot product between a distance matrix and another matrix, which we term the \emph{plan}.
In the DTW case, the plan consists of an alignment matrix, and in the Wasserstein case it consists of a coupling matrix.
Moreover, the optimal coupling $T^*_{ij}$ describes the optimal amount of probability mass to move from point $x_i$ to $y_j$, whilst the optimal alignment $A^*_{ij}$ describes whether or not $x_i$ and $y_j$ are aligned at optimality. While tightly connected, DTW  and the Wasserstein distance between  time series’ support points are still different. For example, if we consider two time series with the same points but reversed ordering, these
 would be far away under DTW, but equal under Wasserstein.

The Wasserstein distance is limited by the requirement for a sensible ground metric $d_{\c{X}}$ to be defined between samples $x_i\in \mathcal{X}$ and $y_j \in \mathcal{Y}$, which is impossible  if there does not exist an explicit correspondence between samples from the compared measures \cite{Solomon16}.
The Wasserstein distance is also not invariant under isometries, such as rotations and translations, and generally leads to a large distance between measures equivalent up to such transformations.
To relax these requirements, \textcite{journals/focm/Memoli11} propose the \emph{Gromov--Wasserstein} (GW) distance between isometry classes of metric-measure triples $(\c{X},d_{\c{X}},\mu_x)$ and $(\c{Y},d_{\c{Y}},\mu_y)$. It is defined as
\[
\begin{aligned}
&\f{GW}(\mu_x,\mu_y)
\\
&=\min_{\m{T} \in \Pi(\v{p},\v{q})} \sum_{ijkl} \mathcal{L}\del[1]{d_{\c{X}}(x_i,x_k),d_{\c{Y}}(y_j,y_l)}T_{ij}T_{kl},
\end{aligned}
\]
where $\mathcal{L}$ is typically squared error loss, and does not rely on a cost or metric to compare $x_i$ with $y_j$.
Instead, GW compares the intra-relational metric geometries of the two measures by comparing the distributions of their pairwise distances. 
This only requires the definition of metrics $d_{\c{X}}$ and $d_{\c{Y}}$ on $\c{X}$ and $\c{Y}$, respectively, which can be arbitrarily different.
GW has been used as a tool for comparing measures on incomparable spaces, notably for training generative models \cite{pmlr-v97-bunne19a}, graph matching \cite{pmlr-v97-xu19b}, and graph averaging \cite{NIPS2019_8569}. 
\textcite{pmlr-v97-titouan19a} also propose \emph{fused Gromov--Wasserstein} to deal with structured objects such as graphs and time series, which consists of a mixture of Wasserstein distance on the node features (for example, time ordering), and GW on the spatial structure, which illustrates how these concepts can be mixed and matched as needed in the specific use case.

\section{Gromov Dynamic Time Warping}
\label{sec:gdtw}
Motivated by the connections between DTW and optimal transport described in Sections \ref{sec:dtw} and \ref{sec:ot}, respectively, we introduce a distance between time series $\v{x} \in \c{X}^{T_x}$ and $\v{y} \in \c{Y}^{T_y}$ defined on potentially incomparable compact metric spaces.
We define the \emph{Gromov dynamic time warping} distance between metric-time-series triples $(\c{X},d_{\c{X}},\v{x})$ and $(\c{Y},d_{\c{Y}},\v{y})$ as
\[
\begin{aligned}
&\f{GDTW}(\v{x},\v{y})
\\
&=\min_{\m{A} \in \c{A}(T_x,T_y)} \sum_{ijkl} \mathcal{L} \del[1]{d_{\c{X}}(x_i,x_k),d_{\c{Y}}(y_j,y_l)} A_{ij}A_{kl},
\end{aligned}
\]
where $\mathcal{L}:\R^2 \-> \R^+$ is a loss function measuring the alignment of the pairwise distances.  The first two elements of the metric-time-series triples are omitted to ease notation.
We think of $\mathcal{L}$ as a proxy for measuring the alignment of the time series (e.g., the square error loss $\mathcal{L}(a,b) = (a-b)^2$).
Under the optimal alignment, for any two pairs $(x_i,y_j)$ and $(x_k,y_l)$, if $x_i$ is close to $x_k$ then $y_j$ will tend to be close to $y_l$.

Provided $\mathcal{L}$ is a pre-metric and so induces a Hausdorff topology, GDTW possesses the following properties: 
\begin{enumerate}
    \item[(a)] $\f{GDTW}(\v{x}, \v{y}) \geq 0$, and  $\f{GDTW}(\v{x}, \v{x}) = 0$,
    \item[(b)] $\f{GDTW}(\v{x}, \v{y}) = 0$ if and only if there exists an isometry $\phi:\c{X}\rightarrow \c{Y}$ such that $\phi(\v{x}) = \v{y}$,
    \item[(c)] $\f{GDTW}(\v{y}, \v{x}) = \f{GDTW}(\v{x}, \v{y})$ if and only if $\c L$ is symmetric.
\end{enumerate}
Mirroring DTW, GDTW does not generally satisfy the triangle inequality.
Thus, GDTW is a pre-metric over equivalence classes of $(\c{X},d_{\c{X}},\v{x})$ triples, up to metric isometry.
A formal treatment is given in Appendix \ref{sec:app_theory}.

Some optimal alignments are given in Figure \ref{fig:alignment_exp}. The original version of DTW-GI (rotationally invariant) fails in the translational case, while its translational extension, obtained by subtracting a bias from both time series, works in both cases---here, invariances have to be manually specified. CTW works in both settings, but invariances are also manually specified by the constraints imposed in the optimization of the learned spatial projections. 
GDTW recovers the correct alignments in both cases without explicitly specifying the symmetries. 

\begin{algorithm}[t]
\caption{Frank--Wolfe-inspired algorithm for Gromov DTW}
\label{alg:gdtw}
\begin{algorithmic}
\STATE Initialize $\m{A}\in \c{A}(T_x,T_y)$ arbitrarily, and compute $L_{ijkl}=\c{L}\del[1]{d_{\c{X}}(x_i,x_k),d_{\c{Y}}(y_j,y_l)}$.
\WHILE{iter $<$ max\_iter and has not converged}
\STATE Update $\m{A} \leftarrow \argmin^\gamma_{\m{A}' \in \c{A}(T_x,T_y)} \innerprod{\m{L} \ox \m{A}}{\m{A}'}_{\f{F}}$ using \eqref{eqn:dynamic-programming} if $\gamma=0$ or \eqref{eqn:argmin_sdtw} if $\gamma>0$.
\ENDWHILE 
\RETURN $\m{A}$
\end{algorithmic}
\end{algorithm}

\subsection{A Frank--Wolfe-inspired Algorithm}

We now present a straightforward and efficient algorithm for computing GDTW.
Following ideas proposed in the optimal transport setting for computing the Gromov--Wasserstein distance, one can introduce a 4-dimensional array $L_{ijkl}=\mathcal{L}\del[1]{d_{\c{X}}(x_i,x_k),d_{\c{Y}}(y_j,y_l)}$ and express GDTW as
\[
\label{eqn:gdtw_tensor}
\f{GDTW}(\v{x},\v{y}) &= \min_{\m{A} \in \c{A}(T_x,T_y)} \mathcal{G}_{\v{x},\v{y}}(\m{A}),
\\
\mathcal{G}_{\v{x},\v{y}}(\m{A}) &= \innerprod{\m{L} \ox \m{A}}{\m{A}}_{\f{F}}
.
\]
This expression is similar to the DTW objective in~\eqref{eqn:dtw}, but with a cost function $\m{D}$ that now depends on the alignment matrix $\m{A}$. 

The Frank--Wolfe (FW) method is an algorithm for solving constrained optimization problems without requiring projections onto the constraint set.
While FW optimization on convex domains has been deeply studied for both convex \cite{doi:10.1002/nav.3800030109,jaggi13} and non-convex  \cite{DBLP:journals/corr/Lacoste-Julien16} objectives,  FW on non-convex domains is largely unexplored.
Inspired by the non-convex Frank--Wolfe algorithm introduced in \textcite{nonconvexfw}, we propose a variant that enforces feasibility of proposals by setting the step size to $1$. Our algorithm consists of the following steps. First, we (i) solve a linear minimization oracle 
\[
\m{S}^{(t)}&=\argmin_{\m{A} \in \c{A}(T_x,T_y)} \innerprod{\nabla_{\m{A}} \c{G}_{\v{x},\v{y}}(\m{A}^{(t)})}{\m{A}}\\&=\argmin_{\m{A} \in \c{A}(T_x,T_y)} \innerprod{\m{L} \ox \m{A}^{(t)}}{\m{A}},
\]
which can be performed exactly in $O(T_xT_y)$ by a DTW iteration, noting that $\m{L}\otimes \m{A}^{(t)}$ can be computed in $O(T_x^2T_y+T_xT_y^2)$ time in the case $\mathcal{L}=L_2$ \cite{pmlr-v48-peyre16}. 
Then, we (ii) updates the iterates. For the step size $\eta^{(t)} = 1$, the update is
\[
\m{A}^{(t+1)}=\m{A}^{(t)}+\eta^{(t)} (\m{S}^{(t)}-\m{A}^{(t)}))= \m{S}^{(t)}.
\]
Keeping step sizes $\eta^{(t)}$ in $\{0,1\}$ remediates the non-convexity of the constraint set, as iterates are guaranteed to remain in $\c{A}(T_x,T_y)$ in spite of non-convexity.

\begin{figure}
    \centering
    \includegraphics[scale=0.65]{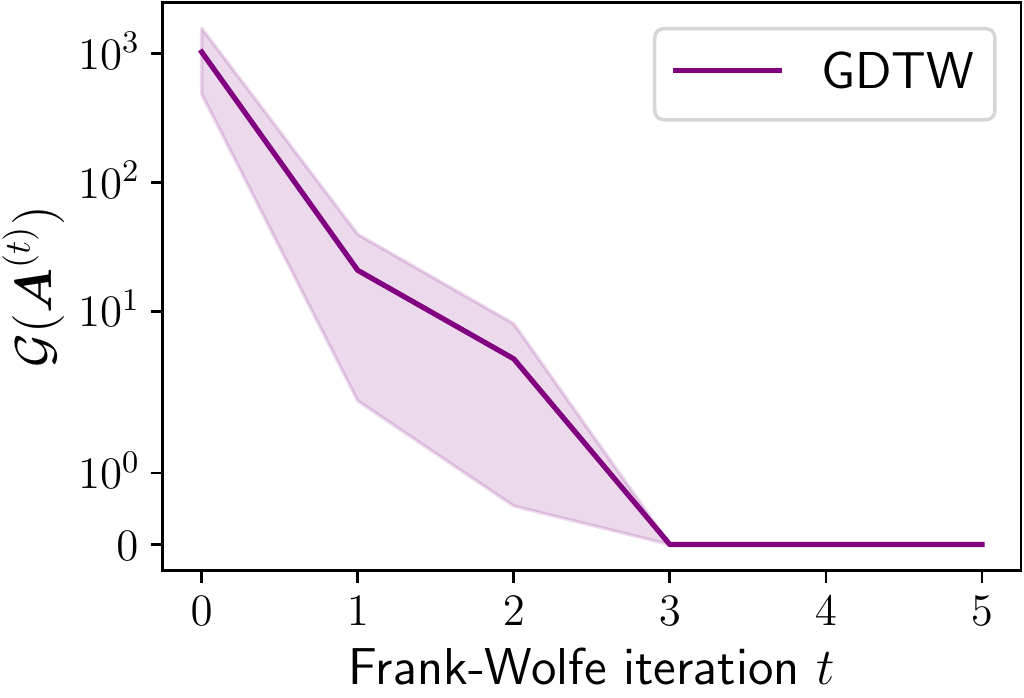}
    \caption{Evolution of the Gromov DTW objective with respect to iteration number for the time series of Figure \ref{fig:alignment_exp}. We plot mean and standard deviation across 10 runs with randomly initialized alignment matrices.}
    \label{fig:loss_values}
\end{figure}

In Figure \ref{fig:loss_values}, we plot the objective $\c{G}_{\v{x},\v{y}}(\m{A}^{(k)})$ at each iteration $k$ across various initializations of alignment matrices, for the time series illustrated in the top row of Figure \ref{fig:alignment_exp}. We observe that in this example, the algorithm recovers the optimal alignment with loss value $0$ in a handful of iterations and is robust with respect to to initialization.

Due to the discrete nature of alignment matrices in the GDTW objective, providing convergence guarantees is non-trivial.
We thus focus on empirical evaluation in Section \ref{sec:experiments} across various settings (such as barycentric averaging, generative modeling, and imitation learning) to demonstrate that the method works well in practice, and defer convergence analysis to future work. 
In practice, we terminate Algorithm \ref{alg:gdtw} if it converges, potentially to a limit cycle, or if the number of iterations reaches a fixed threshold.
A number of alternative algorithms are possible and could be developed, for instance through solving the inner minimization oracle on the convex hull of $\c{A}$ and projecting the results onto the constraint set---we defer these to future work.

\begin{figure}
    \centering
    \includegraphics[scale=0.65]{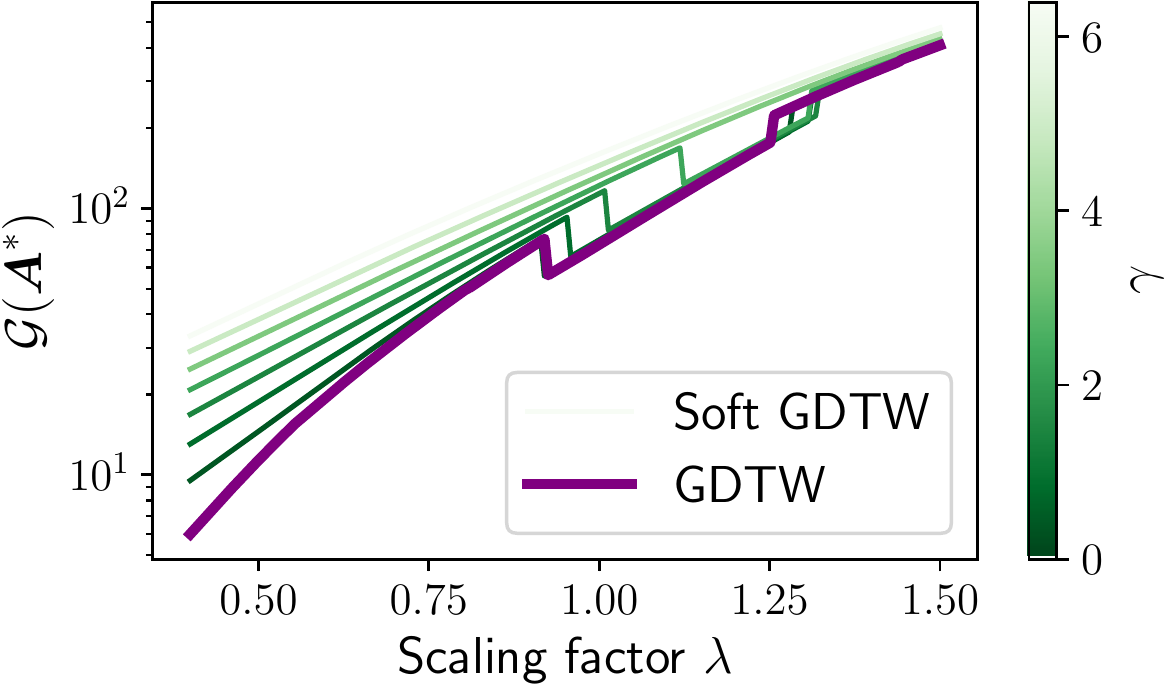}
    \caption{Evolution of (soft) $\f{GDTW}(\v{x},\v{x}_\lambda)$, where $\v{x}_\lambda$ is obtained by distorting the first $T/2$ points of $\v{x}$ by $\lambda$. As $\gamma\to\infty$, soft GDTW becomes smoother and the jumps disappear. As $\gamma\to0$, it converges to GDTW.  }
    \label{fig:differentiability}
\end{figure}

\subsection{Gromov DTW as a Loss Function}

\begin{figure*}
    \centering
    \subfigure[Rotation]{
  \includegraphics[scale=0.95]{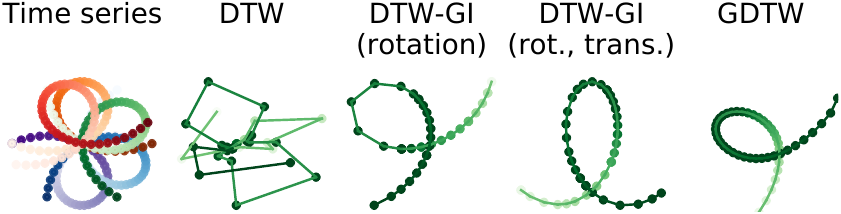}
  \label{fig:bary_curl1}
  }\hspace{0.3cm}
      \subfigure[Rotation and translation]{
   \includegraphics[scale=0.95]{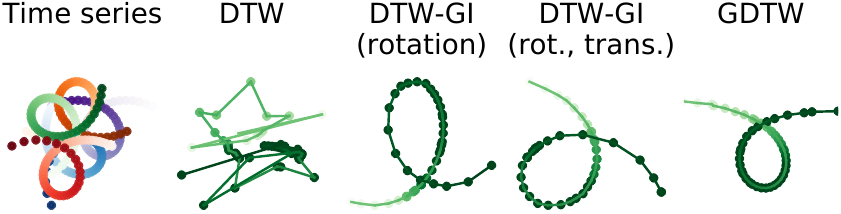}
  \label{fig:bary_curl2}
  }
    \caption{Barycenters of times series with DTW, DTW-GI, and GDTW. In \subref{fig:bary_curl1} random rotations are applied to the time series, while in \subref{fig:bary_curl2} random rotations and translations are applied. DTW fails in both settings, DTW-GI (rotation) fails in the translational setting, while DTW-GI (rotation, translation) and GDTW average sensibly in both as they are invariant to both rotations and translations.
 }
    \label{fig:bary_curl}
\end{figure*}

Gromov DTW can be itself used as a differentiable loss function.
Here, we apply the envelope theorem
\cite{RePEc:mtp:titles:0262531925,milgrom2002envelope} to
\eqref{eqn:gdtw_tensor} and obtain
\[
\label{gdtw_deriv}
\!\nabla_{\v{x},\v{y}} \f{GDTW}(\v{x},\v{y}) &=  \nabla_{\v{x},\v{y}} \innerprod{\m{L}(\v{x},\v{y})\otimes \m{A}^*}{\m{A}^*}_{\f{F}},
\\
\label{gdtw_deriv_2}
\m{A}^* &= \argmin_{\c{A}(T_x,T_y)}\c{G}_{\v{x},\v{y}}(\m{A}).
\]
Similarly to DTW, GDTW suffers from unpredictability when the time series is close to a change point of the optimal alignment matrix because of the discontinuity of derivatives.
To remediate this, we describe how GDTW can be softened analogously to soft DTW, to obtain smoother derivatives. A smoother landscape also helps robustify GDTW with respect to alignment initialization.
The algorithm for computing Gromov DTW consists of successive DTW iterations.
Following ideas from the Gromov--Wasserstein literature, we replace the DTW operation in the iterations with a softened version, by replacing the argmin by the soft argmin in \eqref{eq:softargmin}.
A priori, it may seem that computing this is significantly more involved. 
However, \textcite{pmlr-v70-cuturi17a} observe that
\[
\label{eqn:argmin_sdtw}
\operatorname*{{\argmin}^\gamma}_{\m{A}\in \c{A}(T_x,T_y)} \innerprod{\m{D}}{\m{A}}_{\f{F}} =\grad_{\m{D}} \f{DTW}_\gamma(\m{D}),
\]
where $\argmin^\gamma$ is the softened $\argmin$ defined in \eqref{eq:softargmin}. Hence, \eqref{eq:softargmin} can be computed by reverse-mode automatic differentiation in quadratic time, and soft GDTW iterations can be performed by plugging in $\m{D}=\m{L}\ox\m{A}$.
We approximate the derivatives of soft GDTW by using the optimal soft alignment matrix and applying \eqref{gdtw_deriv} and \eqref{gdtw_deriv_2}: by the envelope theorem, this approximation becomes exact in the small-$\gamma$ limit.

In Figure \ref{fig:differentiability}, we plot the evolution of GDTW and soft GDTW as one of the 2D time series gets distorted by a factor $\lambda$: $\v{x}_\lambda = \v{x}+(0, \lambda)$. Across a range of $\lambda$ values GDTW's optimal alignment matrices vary in discrete steps, which leads to discontinuous values, and hence discontinuous gradients, around such $\v{x}_\lambda$ values. By contrast, soft GDTW with sufficiently high $\gamma$ values is qualitatively smooth with respect to $\lambda$, which remediates discontinuity of GDTW's gradients.

\section{Learning with Gromov DTW as a Loss Function}
\label{sec:learning}
We now present a range of applications of Gromov DTW, including barycentric averaging, generative modeling and imitation learning.

\subsection{Barycenters}\label{subsec:bary}

To compute barycenters of Gromov DTW \eqref{eqn:gdtw_tensor}, we extend the algorithm from \textcite{pmlr-v48-peyre16} to the sequential setting.
Given time series $\v{x}_1,...,\v{x}_J \in \c{X}_1^{T_1},...,\c{X}_J^{T_J}$ and weights $\v{\alpha}\in \Delta_J$, let $(\m{D}_{\v{x}_j})_{mn}=d_{\c{X}_j}(\v{x}^{(m)}_j,\v{x}^{(n)}_j)$.
For fixed $T \in \N$ (length of the barycentric time series), the barycenter is defined as any triple $(\c{X},d_{\c{X}},\v{x})$ satisfying
\[
\label{eq:bary}
\m{D}^* &= \argmin_{\m{D}\in \R^{T\x T}} \sum_{j=1}^{J}\alpha_j \f{GDTW}(\m{D},\m{D}_{\v{x}_j}),
\\
\quad \m{D}_{mn} &= d_{\c{X}}(\v{x}^{(m)},\v{x}^{(n)}),\,\, n,m=1,\dots,T,
\]
where, to ease notation, we denote $\f{GDTW}$ purely in terms of distance matrices.
The barycentric time series can then be reconstructed by applying multi-dimensional scaling (MDS)~\cite{Kruskal:1978eu} to $\m{D}^*$: see Figure \ref{fig:bary_curl} for an illustration.
We rewrite~\eqref{eq:bary}~as
\[
\min_{\substack{\m{D} \in \mathbb{R}^{T\times T}\\\m{A}_1,..,\m{A}_J\in \c{A}(T_x,T_y)}} \sum_{j=1}^{J}\alpha_j \innerprod{\mathcal{L}(\m{D},\m{D}_{\v{x}_j}) \ox \m{A}_j}{\m{A}_j}_{\f{F}}
\label{eq:bary_alter}
\]
and solve it by alternating between minimizing over $\m{A}_j$ for $j \in 1,...,J$ via Algorithm~\ref{alg:gdtw}, and minimizing over $\m{D}$ for fixed  $\m{A}_j$. The latter step admits a closed-form solution given as follows.

\begin{proposition}
\label{prop:barycenters} 
If $\c{L}$ is squared error loss, the solution to the minimization in~\eqref{eq:bary_alter} for fixed $\m{A}_j$ is
\[
\m{D}=\sum_{j=1}^J \frac{\alpha_j \m{A}_j^T\m{D}_{\v{x}_j}\m{A}_j}{  \sum_{j=1}^J \alpha_j(\m{A}_j \v{1})(\m{A}_j \v{1})^T},
\]
where division is performed element-wise, and $\v{1}$ is a vector of ones.
\end{proposition}

\begin{proof}
Appendix \ref{sec:app_theory}.
\end{proof}

\begin{figure*}
\centering
\begin{tabular}{@{} p{1.6cm} @{} p{1.5cm} @{} p{1.7cm} @{} p{1.5cm} @{} p{1.5cm} @{} p{4ex} @{} p{2cm} @{} p{1.5cm} @{} p{1.7cm} @{} p{1.5cm} @{} p{1.5cm} @{}}
\centering\sffamily\footnotesize & \centering\sffamily\footnotesize DTW & \centering\sffamily\footnotesize DTW-GI (rotation) & \centering\sffamily\footnotesize DTW-GI \newline\clap{(rot., trans.)} & \centering\sffamily\footnotesize GDTW & \centering\sffamily\footnotesize & \centering\sffamily\footnotesize & \centering\sffamily\footnotesize DTW & \centering\sffamily\footnotesize DTW-GI (rotation) & \centering\sffamily\footnotesize DTW-GI \newline\clap{(rot., trans.)} & \centering\sffamily\footnotesize GDTW
\end{tabular}
\\
  \includegraphics[height=1.5cm]{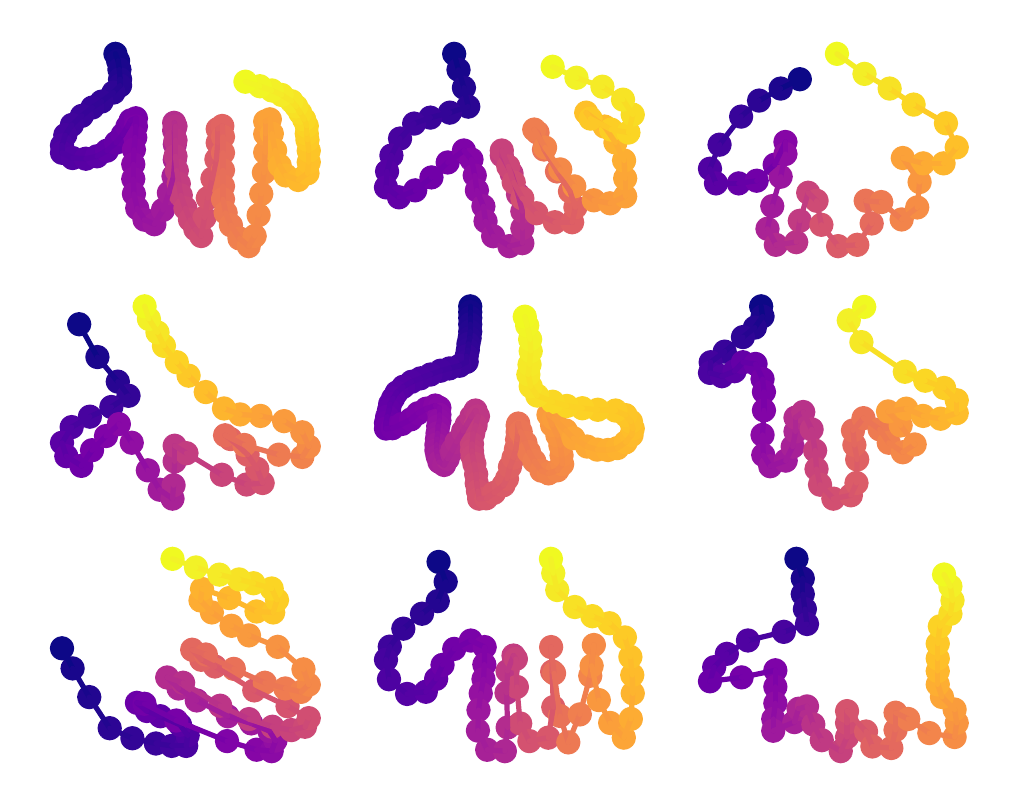}
  \hspace*{-1.75ex}
  \includegraphics[height=1.5cm]{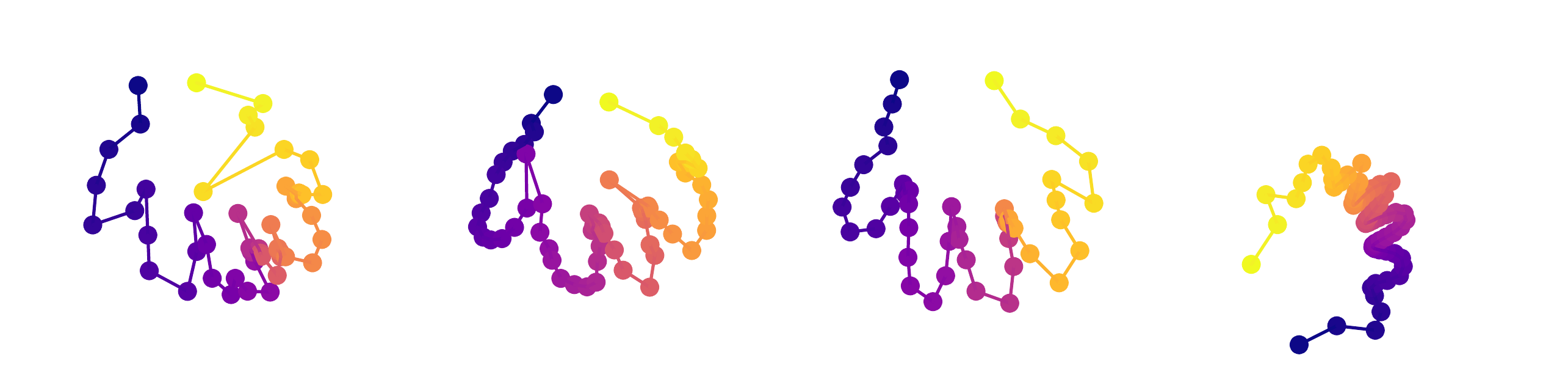}
  \hspace*{1.75ex}
  \includegraphics[height=1.5cm]{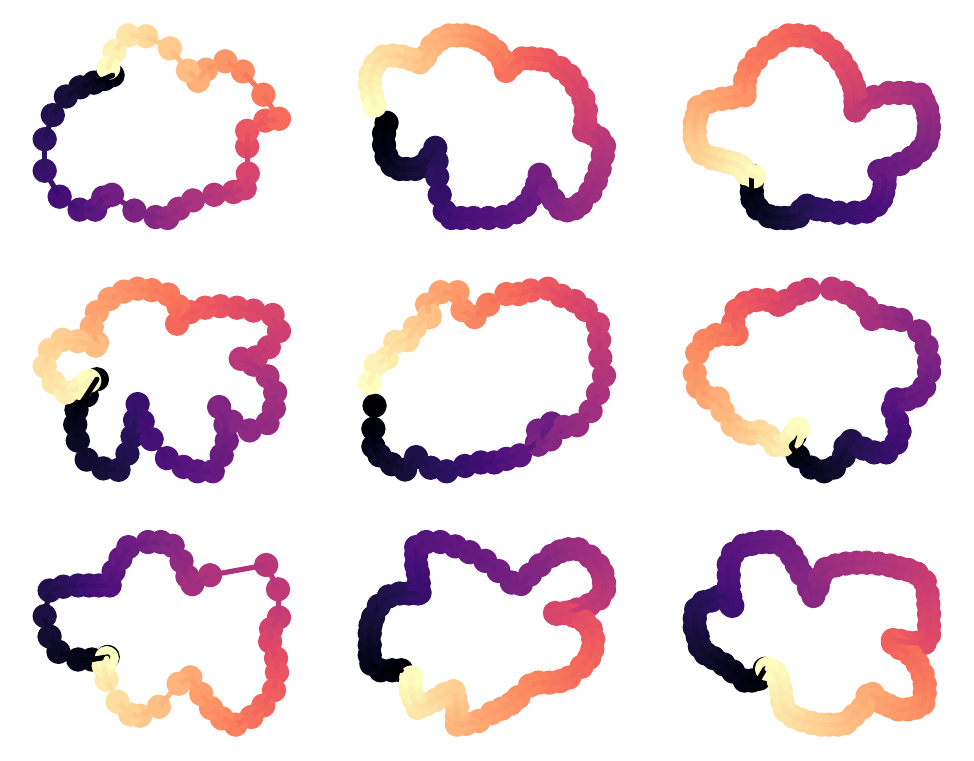}
  \hspace*{-1.75ex}
  \includegraphics[height=1.5cm]{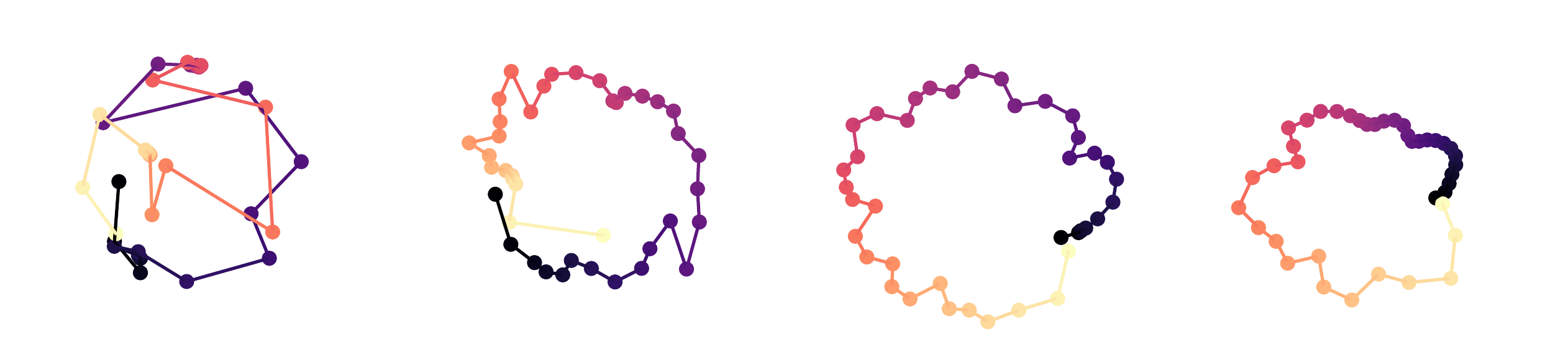}
  \\[2ex]
  \hspace*{-2ex}
  \includegraphics[height=1.5cm]{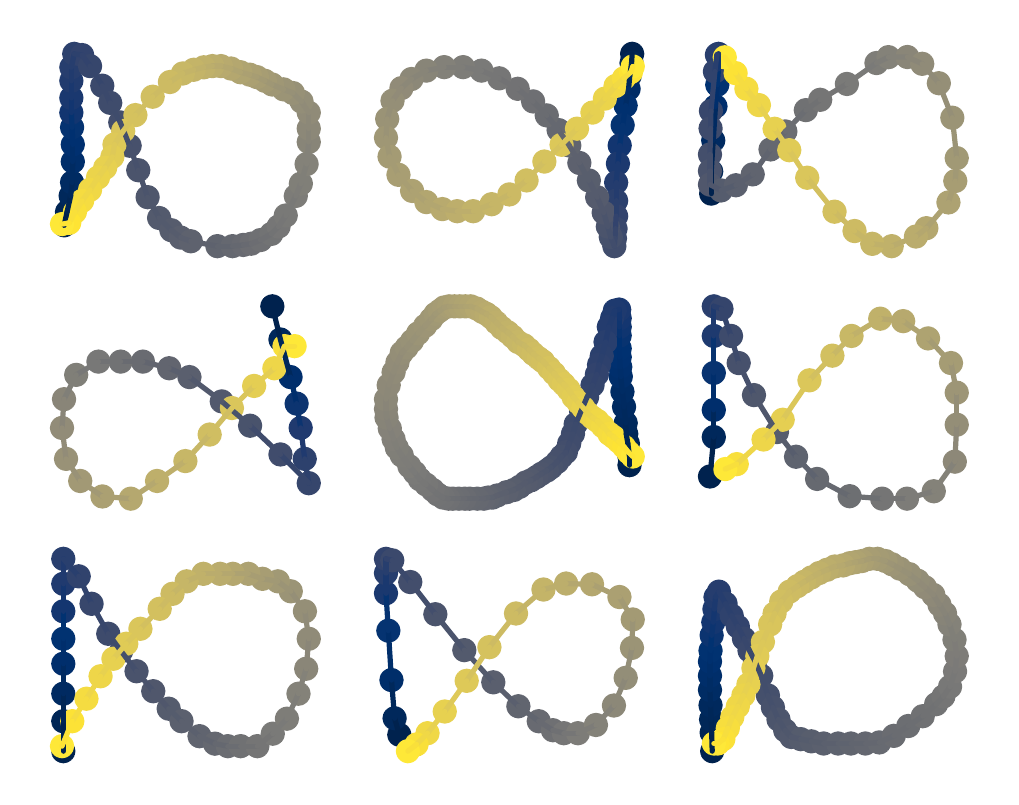}
  \hspace*{-1.75ex}
  \includegraphics[height=1.5cm]{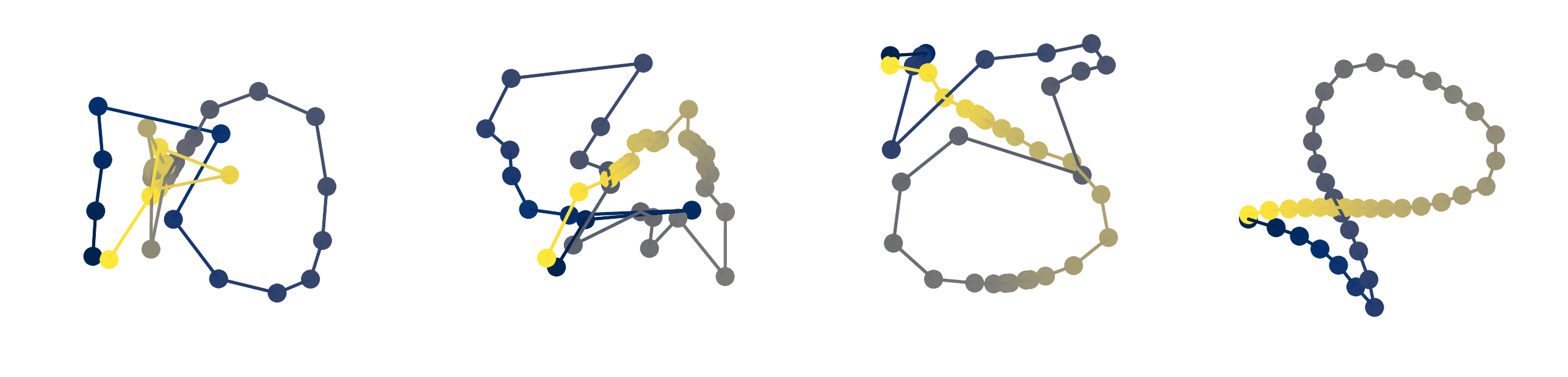}
  \hspace*{1.75ex}
  \includegraphics[height=1.5cm]{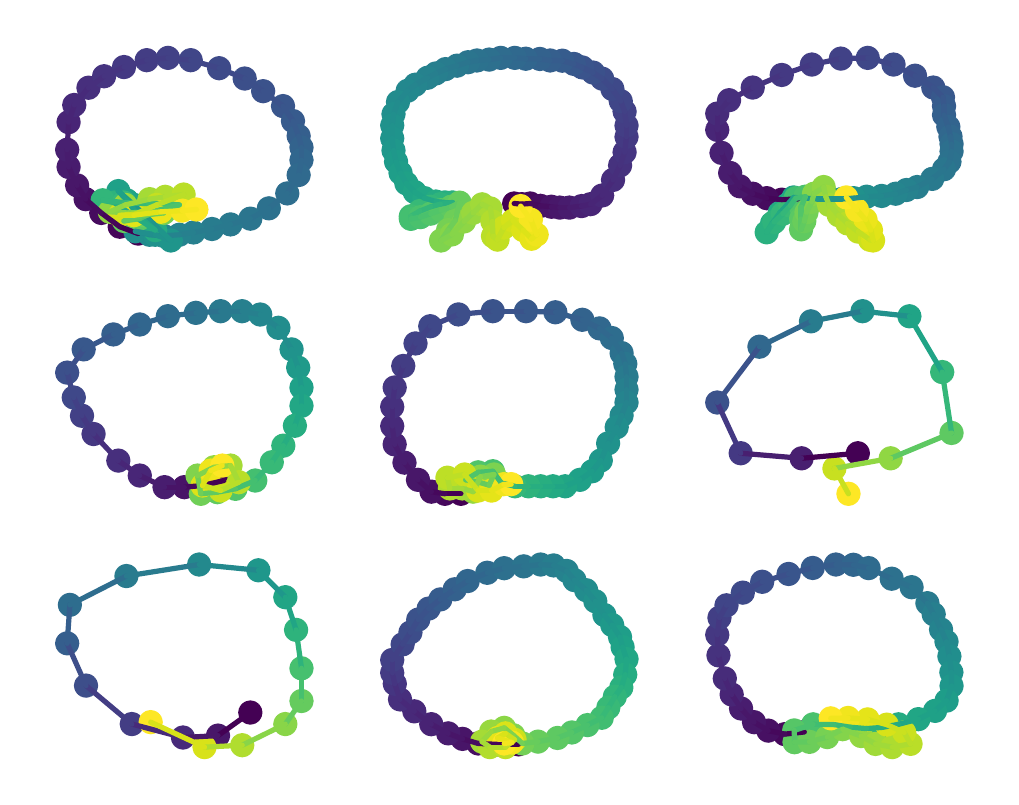}
  \hspace*{-1.75ex}
  \includegraphics[height=1.5cm]{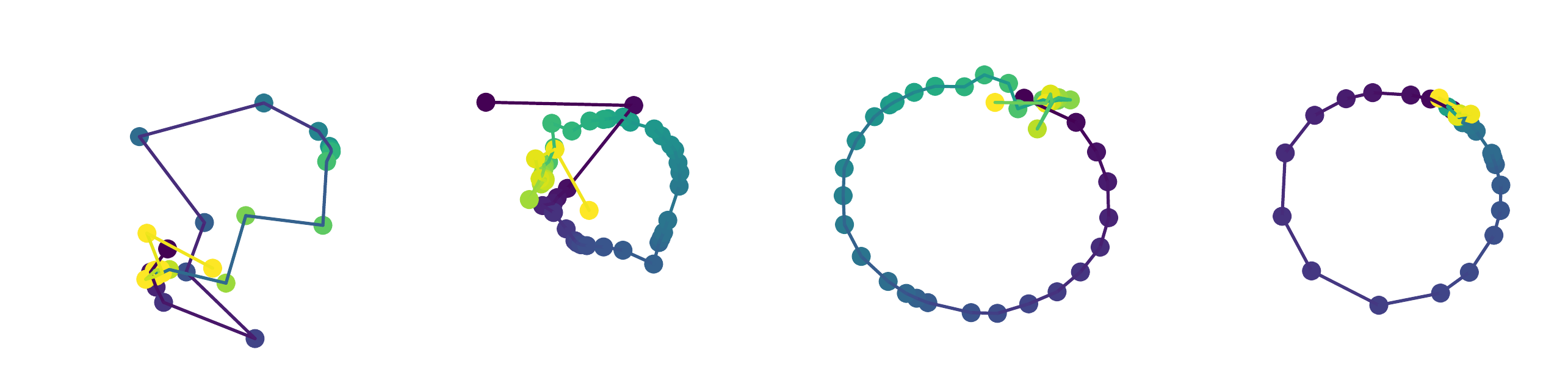}
\caption{Barycenters computed on the QuickDraw dataset using DTW, DTW-GI and GDTW, and sample data points from four different classes (\emph{hands}, \emph{clouds}, \emph{fishes}, \emph{blueberries}). We observe that only GDTW barycenters are meaningful across all datasets, and hence that GDTW better captures the geometric shape of the time series. }
\label{fig:quickdraw_bary}
\end{figure*}

\begin{figure*}[b!]
    \centering
    \begin{minipage}{0.03\hsize}
    \smash{\rotatebox[origin=c]{90}{
      \begin{tabular}{@{} p{1.625cm}@{} p{1.625cm} @{}}
        \centering\sffamily\footnotesize GDTW & \centering\sffamily\footnotesize DTW
      \end{tabular}
      }}
    \end{minipage}
    \begin{minipage}{0.87\hsize}
      \includegraphics[width=\hsize]{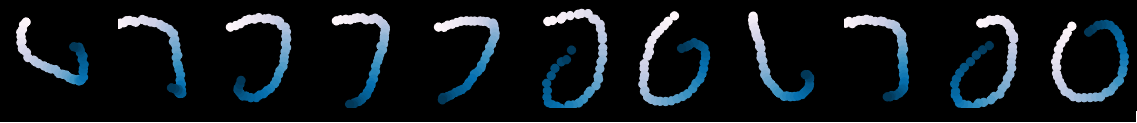}
      \\[-0.5ex]
      \includegraphics[width=\hsize]{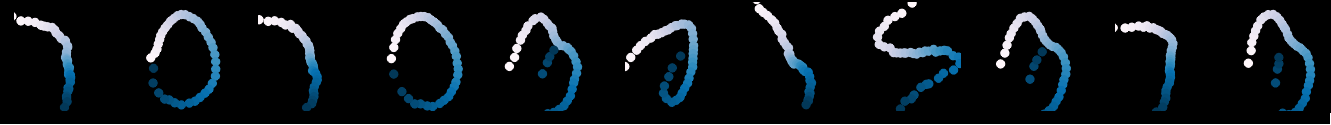}
    \end{minipage}
    \caption{Samples generated by the time series GAN trained on Sequential MNIST, with $\f{DTW}_\gamma$  and $\f{GDTW}_\gamma$, respectively, used as ground costs.}
    \label{fig:gan}
\end{figure*}

\subsection{Generative Modeling}\label{sec:genmod}
We now use GDTW as an approach for training generative models of time series.
Here, we view our dataset of time series $\v{x}^1,...,\v{x}^J \in \c{X}_1^{T_1},...,\c{X}_J^{T_J}$ as a discrete measure $\mu=\frac{1}{J}\sum_{j=1}^J\delta_{\v{x}^j}$.
We define a generative model $\mu_\theta=G_{\theta\#}\nu$, where $\nu$ is a latent measure, such as an isotropic Gaussian, $G_\theta:\c{Z} \-> \c{X}^T$ is a neural network and $G_{\theta\#}\nu$ is the pushforward measure.
By nature of Gromov DTW, the generated time series do not have to live in the same space as the data. 
In particular, this allows us to specify the length of the time series we wish to generate.
We train the model $\mu_\theta$ by minimizing the entropic Wasserstein distance $\f{W}_\eps$ \cite{NIPS2013_4927} between $\mu$ and $\mu_\theta$. 
For the ground cost $d$ of $\f{W}_\eps$, we use $\f{DTW}_\gamma$ and $\f{GDTW}_\gamma$.
For $\f{GDTW}_\gamma$, the objective~is
\[
\label{eq:gdtwgan}
\begin{aligned}
& \min_{\theta \in \Theta} \f{W}_\eps (\mu,\mu_\theta)   
\\
&=\min_{\pi \in \Pi(\mu,\mu_\theta)} \E_{(\v{x},\v{y}) \~ \pi}\f{GDTW}_\gamma(\v{x},\v{y})-\eps H(\pi),
\end{aligned}
\]
where $H$ is the  entropic regularization term.
Following \textcite{pmlr-v84-genevay18a}, it is also possible to use the debiased analog of \eqref{eq:gdtwgan}.
$\f{W}_\eps(\mu,\mu_\theta)$ is computed efficiently using the Sinkhorn algorithm \cite{Sinkhorn1974DiagonalET,NIPS2013_4927}, and $\theta$ is minimized by gradient descent.
This approach extends the Sinkhorn GAN by \textcite{pmlr-v84-genevay18a} and the GWGAN by \textcite{pmlr-v97-bunne19a} to sequential data.

\subsection{Imitation Learning}
\label{sec:imit}
We consider an imitation learning setting in which an agent needs to solve a task given the demonstration of an expert. 
We assume the agent has access to the true transition function $\mathcal{T}$ over the agent's state-space $\c{X}$, and define a state trajectory as a time series $\v{x} \in \mathcal{X}^{T_x}$. 
An expert state trajectory $\v{y}_{\text{exp}} \in \mathcal{Y}^{T_y}$ solving a specific task, such as traversing a maze, is given. 
The goal is to train the agent's parametrized policy $\pi_{\theta}: \c{X} \-> \c{A}$ to solve the given task by imitating the expert's behavior, where $\c{A}$ is the action space. 
To find this policy, the agent uses the model of the environment to predict state trajectories $\v{x}_\theta$ under the current policy $\pi_\theta$, compares these predictions with the expert's trajectory $\v{y}_{\text{exp}}$, and then optimizes the controller parameters $\theta$ to minimize the distance between predicted agent trajectory and observed expert trajectory. 
Using GDTW, our objective is
\[
\label{eq:il_objective}
\min_{\theta} \f{GDTW}_\gamma(\v{y}_{\text{exp}},\v{x}_{\theta})
.
\]
The flexibility of GDTW allows for expert trajectories defined in pixel space $\mathcal{Y}=\mathbb{R}^{32\times 32}$, while the agent lives in $\mathcal{X}=\R^2$. 
Rollouts obtained with $\pi_\theta$ mimic the expert's trajectory up to isometry.
For comparison, instead of \eqref{eq:il_objective}, we also consider DTW. The aim is to learn the same trajectory in the same space as the expert. DTW, in contrast with GDTW, requires $\c{X}=\c{Y}$, and the starting positions for the agent and expert to be close.
From a reinforcement learning perspective, the use of GDTW in \eqref{eq:il_objective} can be interpreted as a value estimate
and gradient-based policy learning as taking estimated value gradients \cite{fairbank2012value,heess2015learning}.
\begin{figure*}[t!]
\centering
\subfigure[$T=1$]{
  \includegraphics[width=0.142\hsize]{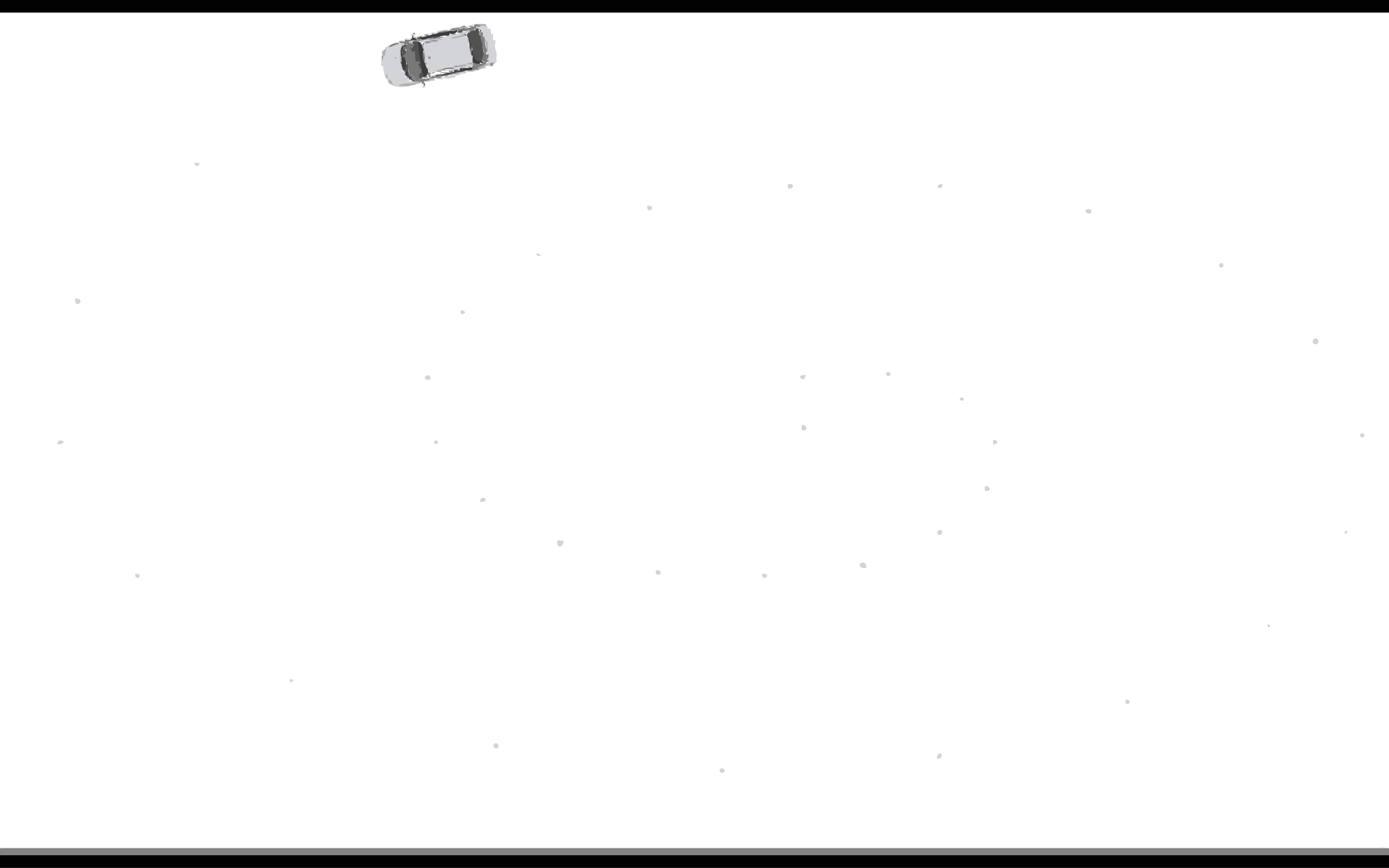}
  \label{fig:1}
  }
  \subfigure[$T=7$]{
  \includegraphics[width=0.142\hsize]{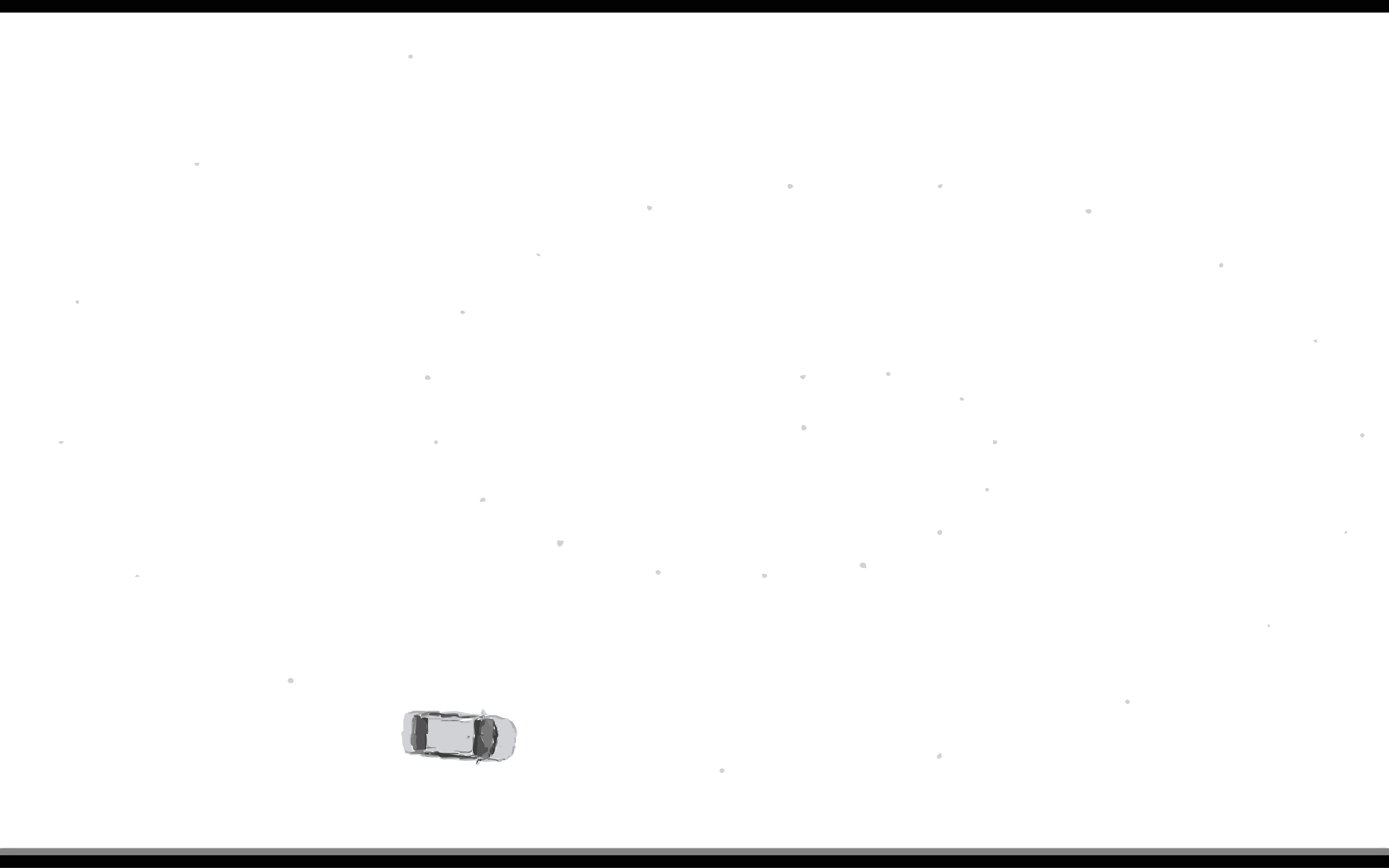}
  \label{fig:7}
  }
  \subfigure[$T=15$]{
  \includegraphics[width=0.142\hsize]{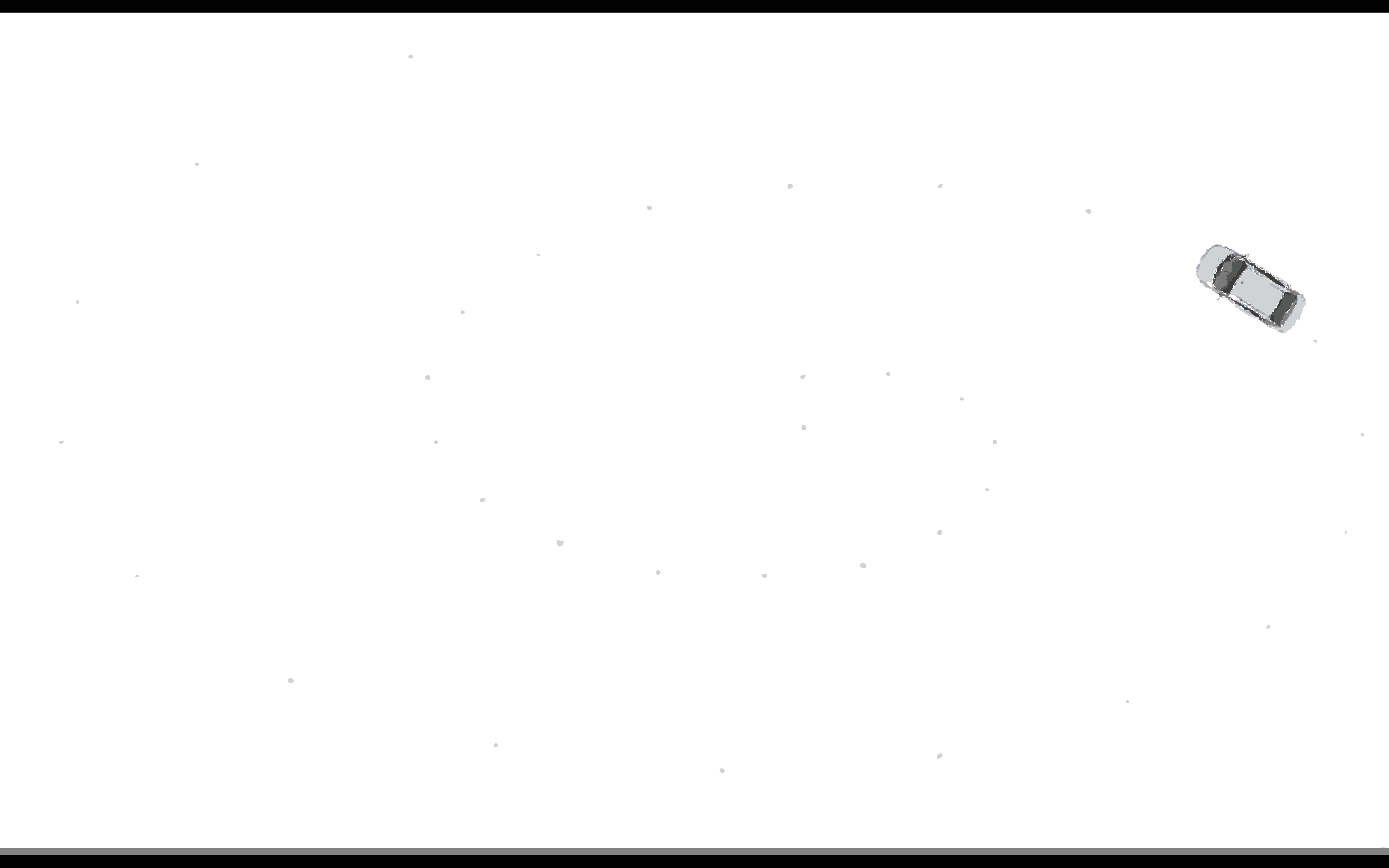}
  \label{fig:15}
  }
  \subfigure[$T=22$]{
  \includegraphics[width=0.142\hsize]{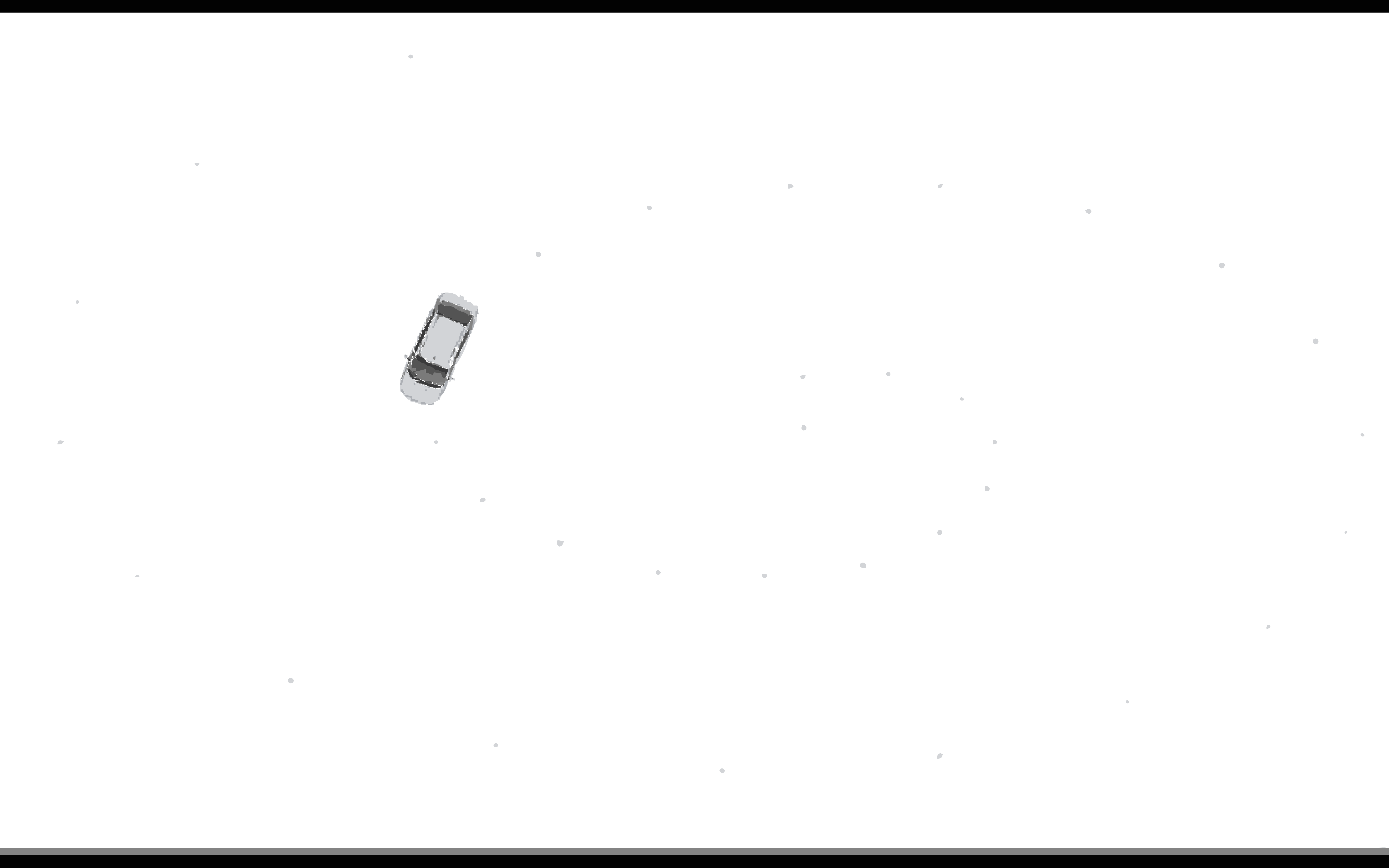}
  \label{fig:22}
  }
  \subfigure[$T=30$]{
  \includegraphics[width=0.142\hsize]{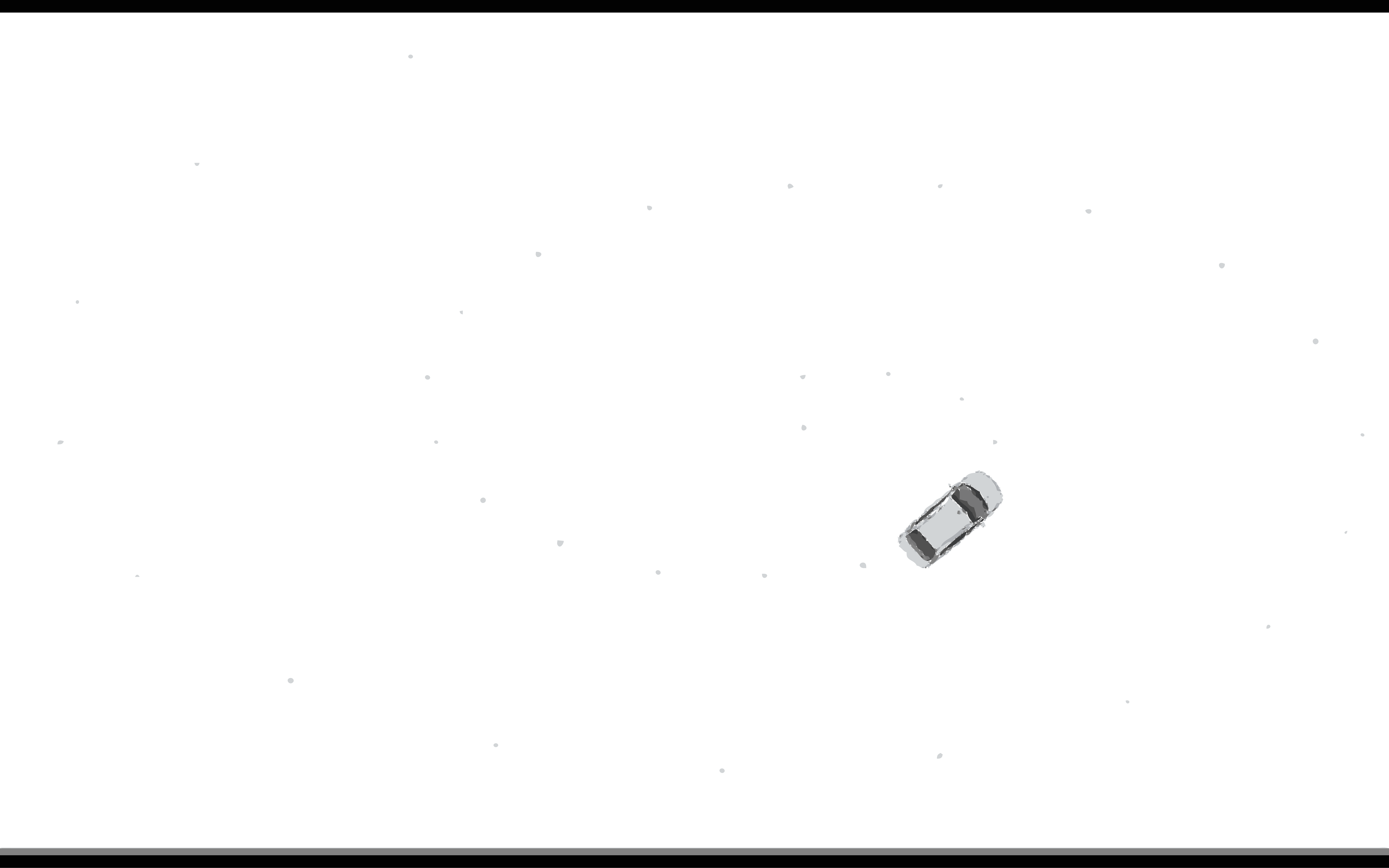}
  \label{fig:30}
  }
  \subfigure[$T=36$]{
  \includegraphics[width=0.142\hsize]{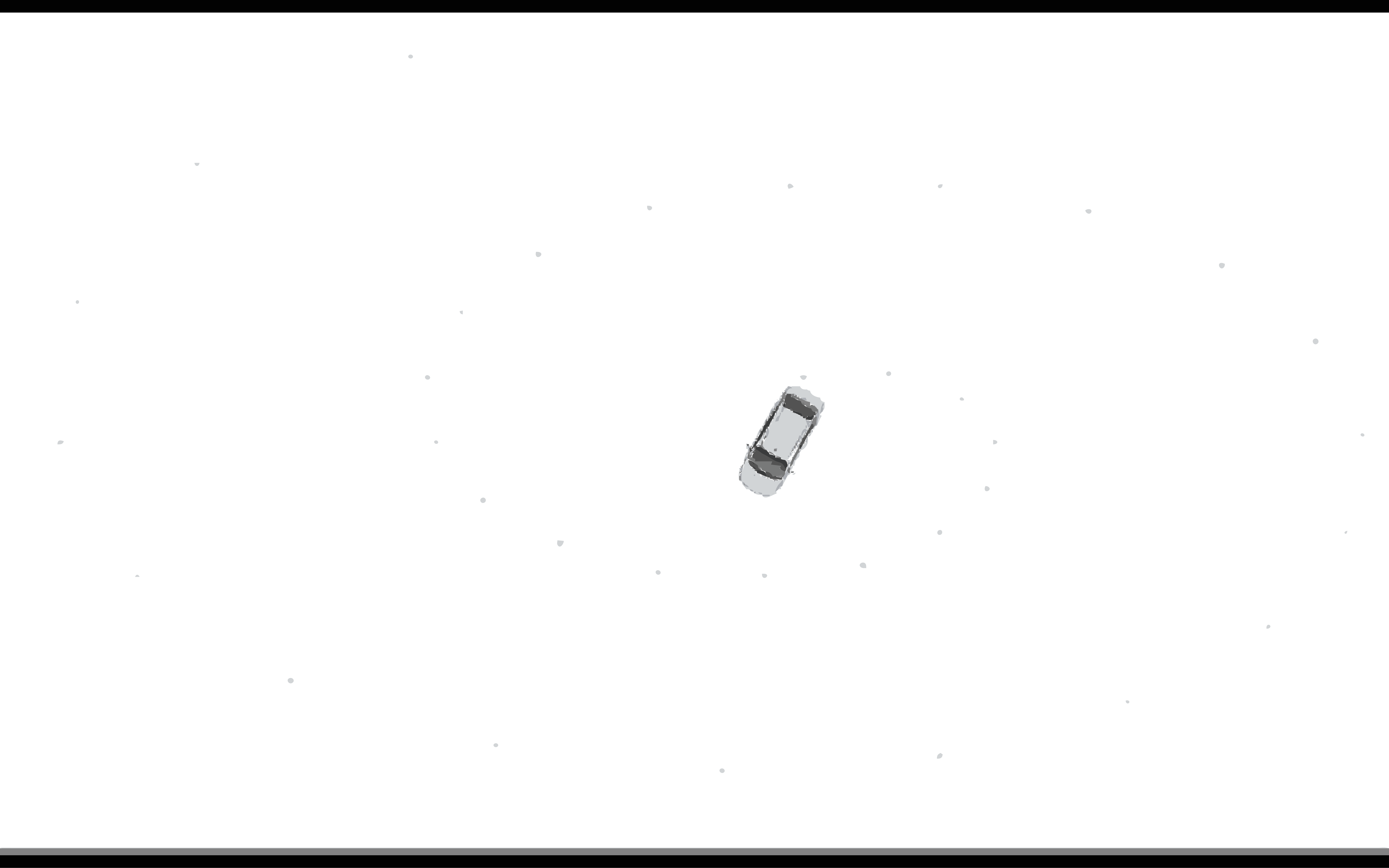}
  \label{fig:36}
  }
\subfigure[Rollout of the learned policy]{
  \includegraphics[scale=0.68]{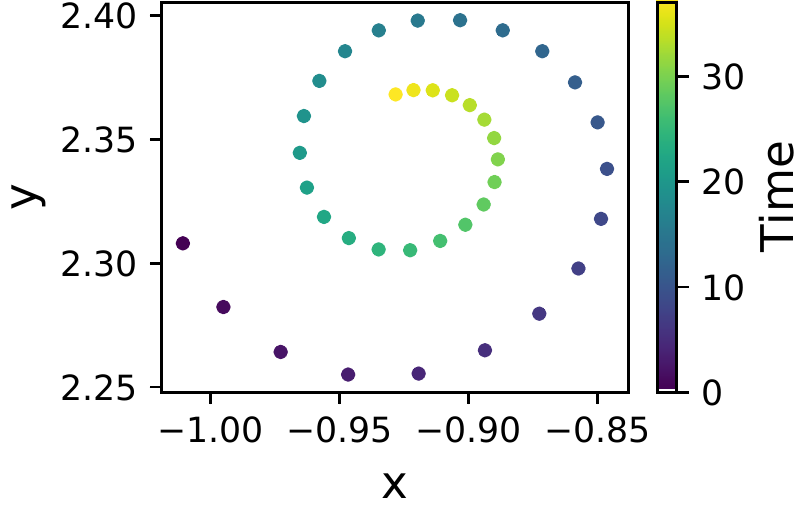}
  \label{fig:policy}
  }
  \subfigure[Loss: video trajectory]{
  \includegraphics[scale=0.68]{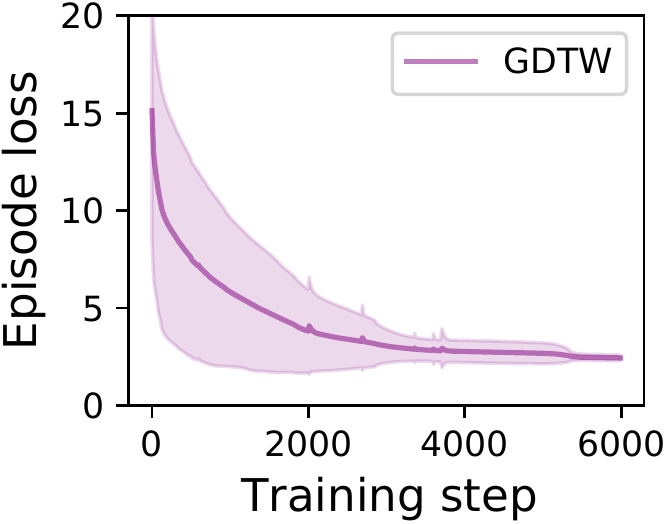}
  \label{fig:rewardincomp}
  }
  \subfigure[Loss: 2D expert trajectory]{
  \includegraphics[scale=0.68]{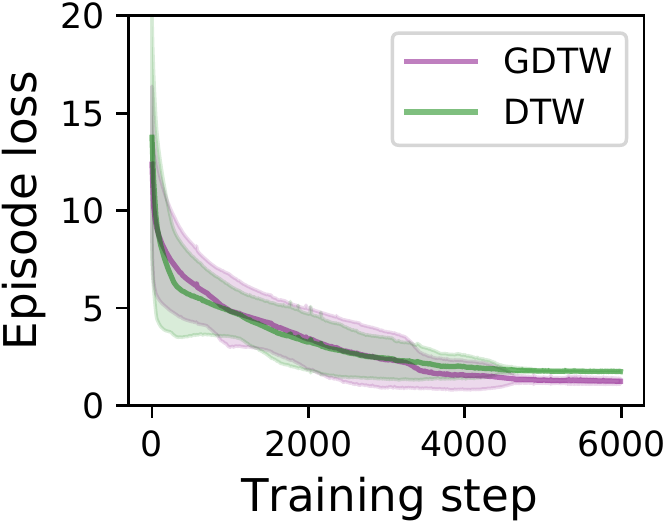}
  \label{fig:rewardcomp}
  }
\caption{\subref{fig:1}--\subref{fig:36}: Snapshot of an expert trajectory (sequence of pixel images); \subref{fig:policy}: policy of an agent in $\R^2$ learned by imitation learning given video demonstrations; \subref{fig:rewardincomp}: log-episodic loss per training step in the video/2D setting; \subref{fig:rewardcomp} in the 2D/2D setting (averaged across 20 seeds, with standard deviations.}
\label{fig:mn1}
\end{figure*}

\begin{figure*}[b!]
\vspace*{-1ex}
\footnoterule
\footnotesize
\quad \footnotemark[1]QuickDraw can be found at \url{https://quickdraw.withgoogle.com/}.
\\
\strut\quad \footnotemark[2]Sequential MNIST can be found at \url{https://github.com/edwin-de-jong/mnist-digits-stroke-sequence-data}.
\label{ftn:code}
\end{figure*}

\section{Experiments}
\label{sec:experiments}

We assess the effectiveness of our proposals in settings in which (i) time series live in comparable spaces and where previous approaches apply, (ii) the spaces are incomparable.

\textbf{Baselines.} Throughout the experiments, we compare $\f{GDTW}_\gamma$ to, in settings in which they apply, $\f{DTW}_\gamma$ \cite{saoke78,pmlr-v70-cuturi17a} its respectively rotationally-invariant and translationally-rotationally-invariant extensions $\f{DTW-GI}$ (rotation), $\f{DTW-GI}$ (rotation, translation) \cite{vayer2020time}, and canonical time warping \cite{ctw}.

\subsection{Alignment}

We first evaluate GDTW on alignment tasks.
We consider two settings in which $\v{y}$ is obtained by applying to $\v{x}$ (i) a rotation, and (ii) a translation followed by a rotation.
In Figure \ref{fig:alignment_exp}, we see that GDTW recovers the right alignment in both settings, while DTW-GI with rotation only works in the rotational setting---this can be seen in the top row of Figure \ref{fig:alignment_exp}. DTW-GI with rotation and translation and CTW work in both settings, while ordinary DTW fails in both. We emphasize that CTW and DTW-GI variants are made invariant to the symmetries  by explicitly optimizing manually specified spatial projections, whilst GDTW works in both settings without needing anything to be specified, as GDTW is invariant to symmetries by construction.
Further experiments with soft DTW and GDTW are given in Appendix \ref{sec:apdxexp}.

\subsection{Barycenter Computation}

We investigate barycentric averaging of GDTW, on both toy data and the QuickDraw\hyperref[ftn:code]{\footnotemark[1]} dataset. 
We compare Gromov DTW to DTW and DTW-GI variants, where barycenters from the latter two methods are computed using DTW barycentric averaging \cite{PETITJEAN201276}. 

\paragraph{Toy data.}
In Figure~\ref{fig:bary_curl}, we see that in comparable settings DTW barycenters fail if time series are rotated or translated. 
DTW-GI with rotation is robust to rotation, but fails when applying both rotations and translations, because the translational symmetry is not manually specified.
By contrast, GDTW is robust to both, and leads to meaningful barycenters in all of the given settings.

\paragraph{QuickDraw dataset.}

The QuickDraw dataset consists of time series of drawings in $\R^2$, belonging to 345 categories. 
Among those categories, we selected \emph{hands}, \emph{clouds}, \emph{fishes}, and \emph{blueberries}. 
To address high variability in classes,  we selected input data following a preprocessing routine described in Appendix \ref{sec:apdxexp}.
A sample of the data sets, together with barycenters computed with DTW, DTW-GI, and GDTW is displayed in Figure~\ref{fig:quickdraw_bary}. 
DTW and DTW-GI  with rotation fail to reproduce the shape of the inputs for most classes. DTW-GI with rotation and translation outperforms DTW-GI with rotation, but fails on the \emph{fish} class, while GDTW provides meaningful barycenters across the range of examples. 
GDTW is thus more robust in recovering the geometric shape of the time series, whilst DTW variants are sensitive to isometries.

\subsection{Generative Modeling}

We evaluate the generative modeling proposal of Section \ref{sec:genmod}, and analyze the behavior of the learned model when using DTW and GDTW. 
Here, we consider the sequential-MNIST dataset,\hyperref[ftn:code]{\footnotemark[2]} which consists of time series of digits in $\R^2$ being drawn, and where each time step corresponds to a stroke.
In Figure~\ref{fig:gan}, we see that samples using GDTW as ground cost \eqref{eq:gdtwgan} are of a significantly higher quality than samples using DTW.
This can be explained by the variability in the data set: slight translations significantly affect DTW, but not GDTW.
Note that the GDTW samples are rotated and reflected, since GDTW only produces learned samples up to metric isometries.

\subsection{Imitation Learning}

We now apply Gromov DTW to the imitation learning setting of Section \ref{sec:imit}. 
Here, we are given an expert trajectory $\v{y}_{\f{exp}}$, and our goal is to find a policy $\pi_\theta$, such that the agent's simulated trajectory $\v{x}_\theta$ mimics $\v{y}_{\f{exp}}$.
We consider maze navigation tasks in two settings: (i) both expert trajectories and the agent's domain are $\mathcal{X}=\mathcal{Y}=\R^2$ and (ii) expert trajectories consist of a video sequence of $32\times 32$ images, giving $\c{Y}=\R^{32\times 32}$, whilst the agent's domain is $\c{X}=\R^2$.
In the first setting, DTW and GDTW apply, whilst in the second setting only GDTW can be used.
Figure \ref{fig:rewardcomp} displays the loss \eqref{eq:il_objective}, which is the GDTW distance to the given trajectory, obtained by learning with GDTW and DTW in (i) averaged across 20 seeds.
We see that in this fully-comparable setting, GDTW and DTW recover the spiral trajectory provided by the expert.

Finally, we consider a setting in which an agent living in $\R^2$ is provided with an expert trajectory $\v{y}_{\f{exp}} $ consisting of a video of a car driving through a spiral, illustrated in Figures \ref{fig:1}--\ref{fig:36} (prior to down-scaling the images). 
Here, the state-space of the agent, $\c{X}=\mathbb{R}^2$, differs from the state-space of the expert, $\c{Y}=\mathbb{R}^{32\times 32}$.
The cost on image space $d_{\c{Y}}$ is the $2$-Wasserstein distance, with images interpreted as densities on a grid. 
The cost on the Euclidean space $d_{\c{X}}$ is the Euclidean distance. 
Figure \ref{fig:policy} shows the agent's trajectory under the learned policy $\pi_\theta$, and Figure \ref{fig:rewardincomp} shows the loss \eqref{eq:il_objective} against the number of training steps.
Using GDTW, the agent successfully learns to solve the task despite never having access to trajectories in the space of interest.

\section*{Conclusion}

We propose Gromov DTW, a distance between time series living on potentially incomparable spaces. 
GDTW compares intra-relational geometries of the time series, alleviating the need for a ground metric to be defined on potentially incomparable spaces.
Moreover, GDTW is invariant under isometries by nature, which contributes to its versatility and is an important inductive bias for generalization.
We hope these contributions enable use of time series alignment in novel settings.

\section*{Acknowledgments}
We are grateful to K. S. Sesh Kumar for ideas on the Frank--Wolfe algorithm. SC was supported by
the Engineering and Physical Sciences Research Council (grant number EP/S021566/1).
\printbibliography

\newpage
\appendix
\onecolumn

\section{Theory}
\label{sec:app_theory}

\subsection*{Metric Properties}

Here we develop the theory of Gromov dynamic time warping distances.
We begin by introducing the necessary preliminaries.

\begin{definition}[Time series]
Let $(\c{X},d_{\c{X}})$ be a compact metric space, and let $I_{\c{X}} = \{1,2,..,T_{\c{X}}\} \subset \N$.
We call a finite sequence $\v{x} : I_{\c{X}} \-> \c{X}$ a \emph{time series}.
Let $X$ be the space of all time series.
\end{definition}

\begin{definition}
Let $\v{x}$ and $\v{y}$ be time series.
Define a pre-metric $D : \c{X} \x \c{Y} \-> \R$, which we call the \emph{cost}.
Define the $m \x n$ \emph{cost matrix} $\m{D} \in \R^{m\x n}$ by $D_{ij} = D(x_i, y_j)$.
\end{definition}

\begin{definition}
We say that a binary matrix $\m{A}$ is an \emph{alignment matrix} if $A_{11} = 1$, $A_{mn} = 1$, and $A_{ij} = 1$ implies exactly one of $A_{i-1,j} = 1$, $A_{i,j-1} = 1$, and $A_{i-1,j-1} = 1$ holds.
Let
\[
\c{A} = \cbr{\m{A} \in \{0,1\}^{m \x n} : \m{A} \mathrel{\text{is an alignment matrix}}}
\]
be the set of \emph{alignment matrices}.
\end{definition}

\begin{definition}[Dynamic Time Warping]
Let $\v{x}$ and $\v{y}$ be time series.
Define the \emph{dynamic time warping} distance by
\[
\f{DTW}(\v{x},\v{y}) = \min_{\m{A}\in\c{A}} \innerprod{\m{D}}{\m{A}}_{\f{F}},
\]
where $\innerprod{\cdot}{\cdot}_{\f{F}}$ is the Frobenius norm over real matrices.
\end{definition}

\begin{proposition}
If $D$ is a pre-metric, then $\f{DTW}: X \x X \-> \R$ is a pre-metric on the space of time series. 
If we take $c = d_{\c{X}}$, then $\f{DTW}: X \x X \-> \R$ is a symmetric pre-metric on $X$.
\end{proposition}

\begin{proof}
\textcite{10.1016/j.patcog.2008.11.030}.
\end{proof}

A pre-metric induces a Hausdorff topology on the set it is defined over, and so is suitable for many purposes that ordinary metrics are used for.
To proceed along the path suggested by Gromov-Hausdorff and Gromov--Wasserstein distances over metric-measure spaces, we need to define the time series analog.

\begin{definition}
Define a \emph{metric space equipped with a time series} to be a triple $(\c{X},d_{\c{X}},\v{x})$.
\end{definition}

\begin{definition}
\label{def:gdtw-isom}
Let $(\c{X},d_{\c{X}},\v{x})$ and $(\c{Y},d_{\c{Y}},\v{y})$ be metric spaces equipped with time series.
Define $X|_{\v{x}} = \{x \in X : x \in \f{img}\v{x}\}$, and $Y|_{\v{y}}$ similarly, and equip both sets with their respective subset metrics.
We say that $(\c{X},d_{\c{X}},\v{x})$ and $(\c{Y},d_{\c{Y}},\v{y})$ are \emph{isomorphic} if there is a metric isometry $\phi: X|_{\v{x}} \-> Y|_{\v{y}}$ such that $\phi(\widehat{x}_i) = \widehat{y}_i$, where $\widehat{\v{x}}$ and $\widehat{\v{y}}$ denote $\v{x}$ and $\v{y}$ with consecutive repeated elements removed.
\end{definition}

At this stage it is not clear whether or not the class of all such triples under isometry forms a set, or is instead a proper class.
To avoid set-theoretic complications, we need the following technical result.

\begin{result}
The class of all isometry classes of compact metric spaces is a set.
\end{result}

\begin{proof}
\textcite[ch. 27, p. 746]{alma991005863149705596}.
\end{proof}

It follows immediately that the class of all metric spaces equipped with time series is a set, provided that identification by isometry extends to the time series.
We are now ready to define GDTW.

\begin{definition}
Let $\c{L}$ be a pre-metric on $\R^+$, and define $\c{L} \in \R^{m\x n \x m \x n}$ by
\[
\c{L}_{ijkl} = \c{L}\del[1]{d_{\c{X}}(x_i,x_k), d_{\c{Y}}(y_j,y_l)}
.
\]
Define the \emph{Gromov Dynamic Time Warping} distance by
\[
\f{GDTW}\del[1]{(\c{X}, d_{\c{X}}, \v{x}), (\c{Y}, d_{\c{Y}}, \v{y})} = \min_{\m{A}\in\c{A}} \innerprod{\c{L} \ox \m{A}}{\m{A}}_{\f{F}},
\]
where $(\c{L} \ox \m{A})_{ij} = \sum_{kl} L_{ijkl} A_{kl}$.
\end{definition}

\begin{proposition}
$\f{GDTW}$ is a pre-metric on the set of all metric spaces equipped with time series up to isometry.
\end{proposition}

\begin{proof}
We check the conditions.
Non-negativity is immediate by definition.
It also follows immediately that $(\c{X},d_{\c{X}},\v{x}) \isom (\c{Y},d_{\c{Y}},\v{y})$ implies $\f{GDTW}\del[1]{(\c{X},d_{\c{X}},\v{x}), (\c{Y},d_{\c{Y}},\v{y})} = 0$.
We thus need to prove that $\f{GDTW}\del[1]{(\c{X},d_{\c{X}},\v{x}), (\c{Y},d_{\c{Y}},\v{y})} = 0$ implies $(\c{X},d_{\c{X}},\v{x}) \isom (\c{Y},d_{\c{Y}},\v{y})$.
By hypothesis, we have 
\[
\f{GDTW}\del[1]{(\c{X},d_{\c{X}},\v{x}), (\c{Y},d_{\c{Y}},\v{y})} = \sum_{ijkl} A_{ij} \c{L}_{ijkl} A_{kl} = \sum_{\substack{A_{ij}=1\\A_{kl}=1}} \c{L}_{ijkl},
\]
where all elements of the last sum are non-zero.
Suppose without loss of generality that $\v{x}$ and $\v{y}$ contain no duplicate elements.
We argue inductively that optimal $\m{A}$ is the identity matrix.
\1 First, note that $A_{11} = 1$ by definition of $\m{A}$.
\2 Now, consider $A_{21}$. 
If we suppose $A_{21} = 1$, then we must have $\c{L}_{2111} = 0$, and hence $d_{\c{X}}(x_2,x_1) = d_{\c{Y}}(y_1,y_1) = 0$.
But then $x_2 = x_1$, contradicting the assumption there are no duplicates.
Hence, $A_{21} = 0$.
\3 By mirroring the above argument, $A_{12} = 0$.
Hence, by definition of $\m{A}$, the only remaining possibility is $A_{22} = 1$.
Inductively, we conclude $A_{ii} = 1$ for all $i$, and $A_{ij} = 0$ for $i \neq j$.
\4 
Finally, since the lower-right corner of $\m{A}$ has to also be equal to one by definition, it follows that $\m{A}$ is the square identity matrix.
\0 
Hence $A_{ij} = 1$ and $A_{kl} = 1$ if and only if $i=j$ and $k=l$.
Plugging this into the previous equality yields $d_{\c{X}}(x_i,x_k) = d_{\c{Y}}(y_i,y_k)$ for all $i,k$, which together with diagonal $\m{A}$ gives the isomorphism. 
Finally, to see that lack of duplicates truly is assumed without loss of generality, note that if there are duplicates in $\v{x}$ and $\v{y}$, then we apply the above argument to $\widehat{\v{x}}$ and $\widehat{\v{y}}$ of Definition \ref{def:gdtw-isom}, which no longer contain duplicates.
The claim follows.
\end{proof}

One can easily see that $\f{GDTW}$ will be symmetric if $L$ is symmetric.
Since $\f{DTW}$ itself doesn't satisfy a triangle inequality \cite{10.1016/j.patcog.2008.11.030}, $\f{GDTW}$ won't satisfy it either.

\subsection*{Barycenter Computation}

\begin{proposition}
If $\c{L}$ is a square error loss, the solution to the minimization in~\eqref{eq:bary_alter} for fixed $\m{A}_j$ is
\[
\m{D}=\sum_{j=1}^J \alpha_j \m{A}_j^T\m{D}_{\v{x}_j}\m{A}_j\Big \slash \sum_{j=1}^J \alpha_j(\m{A}_j \v{1})(\m{A}_j \v{1})^T,
\]
where division $\cdot \slash \cdot$ is performed element-wise, and $\v{1}$ is a vector of ones.
\end{proposition}

\begin{proof}
If $\c{L}$ is square error loss, then \eqref{eq:bary_alter} can be written as
\[
\min_{\m{D}} \sum_{j=1}^{J}\alpha_j \innerprod{\m{D}\odot \m{D}\m{A}_j\v{1}\v{1}^T+\v{1}\v{1}^T\m{A}_j\m{D}_{\v{x}_j}\odot \m{D}_{\v{x}_j}-2\m{D}\m{A}_j\m{D}^T_{\v{x}_j}}{\m{A}_j}_{\f{F}},
\]
where $\odot$ is element-wise matrix multiplication. 
Differentiating the objective with respect to $\m{D}$ and setting it equal to $0$, we get
\[
\m{D} \odot \left(\sum_{j=1}^J\alpha_j  (\m{A}_j \v{1})(\v{1}^T \m{A}_j^T)\right)= \sum_{j}\alpha_j \m{A}_j^T \m{D}_{\v{x}_j}\m{A}_j,
\]
which, dividing both sides element-wise, gives the result.
\end{proof}

\section{Experimental Details}
\label{sec:apdxexp}

\subsection*{Alignments}
In Figures \ref{fig:ali_app_1}--\ref{fig:ali_app_4}, we provide further alignment experiments. Note that in this extra set of experiments, we consider the only rotationally invariant proposal of \textcite{vayer2020time}.
Here, we set the entropic term $\gamma$ to $1$ for soft alignments, and we use normalized distance matrices. 
We observe that GDTW and soft GDTW are robust to scaling, rotations and translations, whilst DTW and soft DTW are sensitive to rotations and translations. 
Finally, DTW-GI (rotation) is robust to rotations, but sensitive to translations, which further corroborates the observations from Figure \ref{fig:alignment_exp}. 

\subsection*{Barycenters}
In this experiment, we perform barycenters of 30 elements of 4 quickdraw classes with respect to DTW, DTW-GI and GDTW.

\paragraph{Data selection and pre-processing.} The classes considered in the experiment are \emph{fish}, \emph{blueberries}, \emph{clouds} and \emph{hands}. 
The variability in each class of QuickDraw is extremely high: we created datasets of 30 elements such that it is straightforward to recognize to which category the element belongs to, such that the element is drawn with a single stroke and such that it has a common style. 
The full datasets are displayed in Figure \ref{fig:barycenters}. 
Before running the algorithms, we rescale the data, applying the transformation $\v{x} \mapsto (\v{x} - \min(\v{x})) / \max(\v{x})$ to each data point.
Finally, we down-sample the length of the time series reducing it by 1/3 for \emph{hands} and  1/2 for \emph{fish}, \emph{clouds} and \emph{blueberries}. 

\paragraph{Algorithms.}
For GDTW barycenters, we apply the algorithm of Section \ref{subsec:bary}, using the entropy regularized version of GDTW with $\gamma=1$.
For DTW and DTW-GI, we use standard DBA procedures. 
For both algorithms, we set the barycentric length to 60 for \emph{fish} and \emph{hands} and 40 for \emph{clouds} and \emph{blueberries}. 
We set the maximum number of FW iterations for GDTW to 25, and the number of DTW-GI iterations to 30.

\subsection*{Generative Modeling}
In this experiment, we use the Sinkhorn divergence objective.
We use a latent dimension of $15$, and the generator is a $4$-layer MLP with $1000$ neurons per layers. 
The length of the generated time series is set to $T=40$, and the dimension of the space is $p=2$, thus the MLP's output dimension is $T\times p=80$. 
We set the batch size to $25$. 
We use the ADAM optimizer, with $\v\beta=(0.5,0.99)$, and the learning rate set to $5\times 10^{-5}$. 
We set $\gamma=1$, and the maximum number of iterations in the GDTW computation to $10$.
We use the sequential MNIST dataset\hyperref[ftn:code]{\footnotemark[3]} and normalize the data, which is a time series in $\mathbb{R}^2$, into the unit square. 

\subsection*{Imitation Learning}
In this experiment, we use a two-layer MLP policy, with input dimension of $\dim(\mathcal{X})$, a hidden dimension of 64, and an output dimension of $2$. 
The learning rate is set to $5\times 10^{-5}$, and we use the ADAM optimizer with $\v\beta=(0.5,0.99)$. 
In the video/2D experiment,\hyperref[ftn:code]{\footnotemark[4]} the ground cost for the video is entropic 2-Wasserstein distance, computed efficiently using \textsc{GeomLoss} \cite{feydy2019interpolating}, and the ground cost on the 2D space is squared error loss. We plot mean scores along with standard deviations (across 20 random seeds).

\begin{figure*}[b!]
\vspace*{-1ex}
\footnoterule
\footnotesize
\quad \footnotemark[3]Sequential MNIST can be found at \url{https://github.com/edwin-de-jong/mnist-digits-stroke-sequence-data}.
\\
\strut\quad \footnotemark[4]The video was generated using \url{https://github.com/gezichtshaar/PyRaceGame}.
\label{ftn:code_apdx}
\end{figure*}

\begin{figure}\label{fig:dataset_blueb}
  \centering
  \includegraphics[width=\hsize]{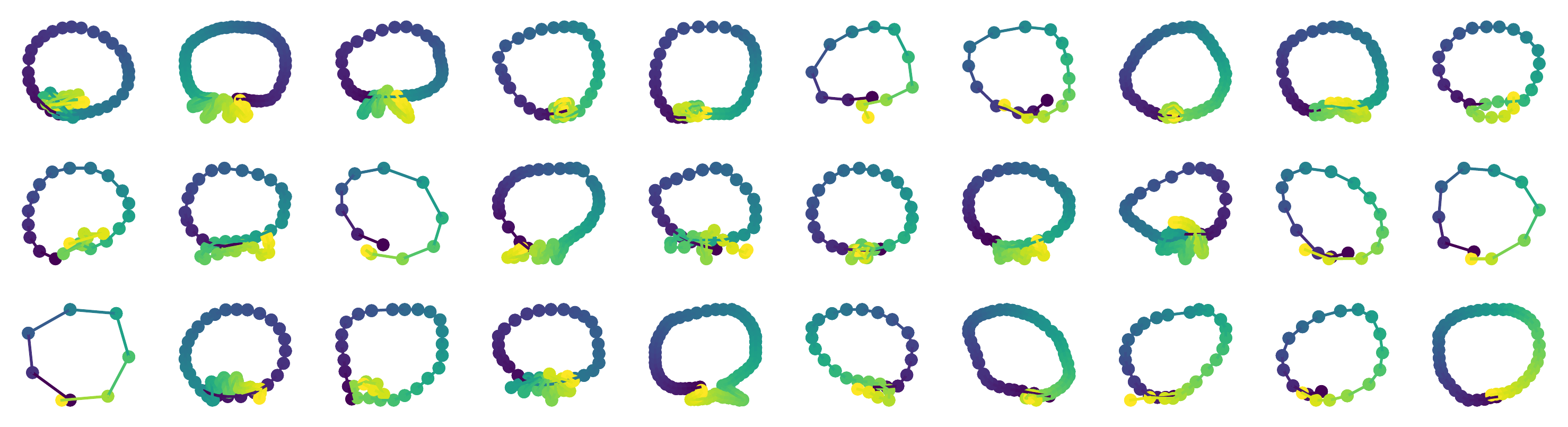}
  \includegraphics[width=\hsize]{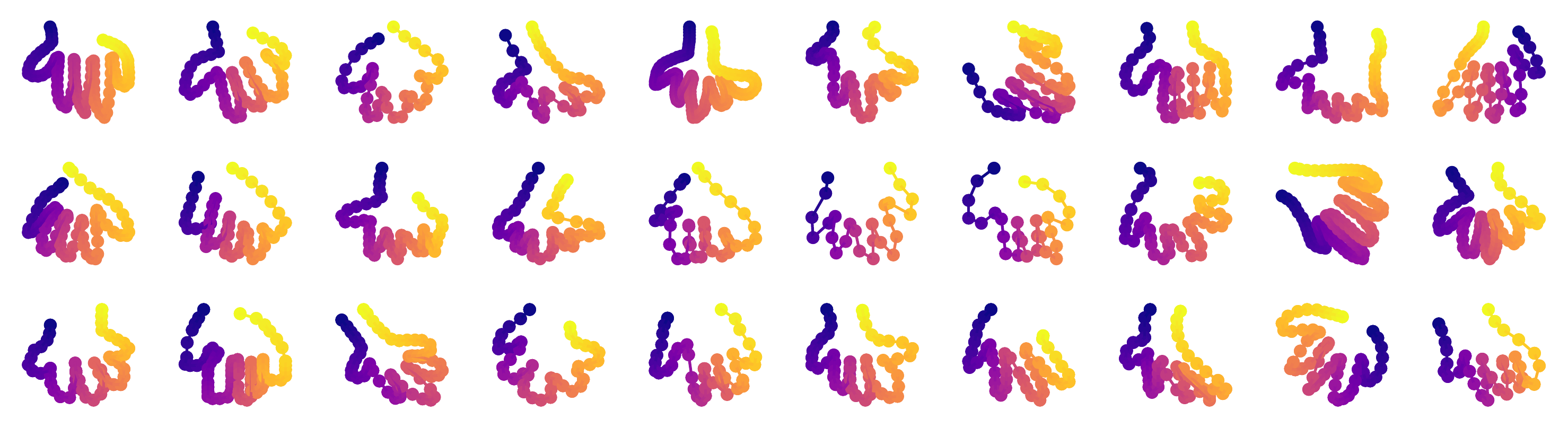}
  \includegraphics[width=\hsize]{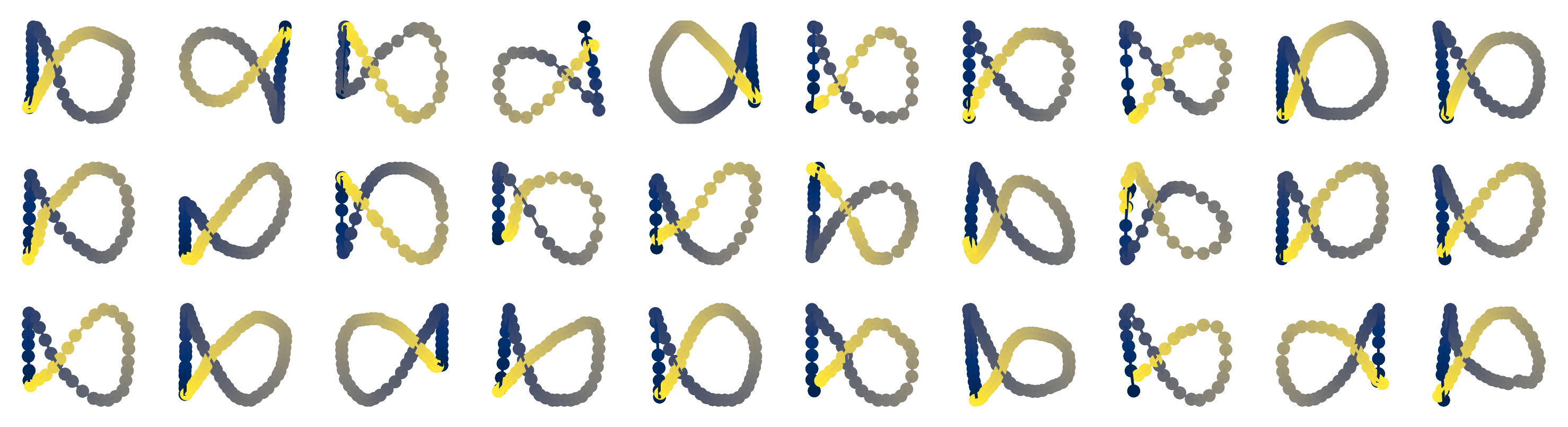}
  \includegraphics[width=\hsize]{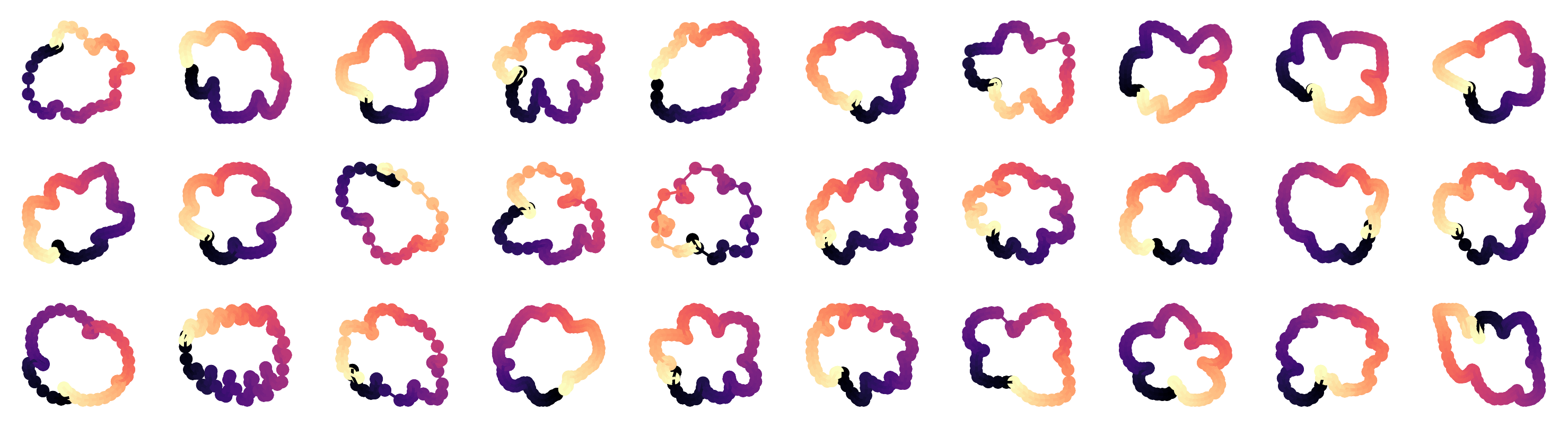}
  \caption{Quickdraw datasets, with classes \emph{blueberries}, \emph{hands}, \emph{fishes}, \emph{clouds}.}
  \label{fig:barycenters}
  \end{figure}

\begin{figure}
    \centering
    \includegraphics{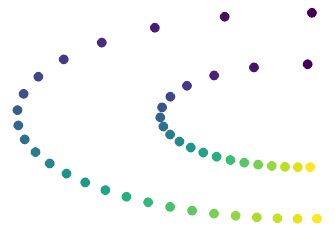}\\[5ex]
    \includegraphics{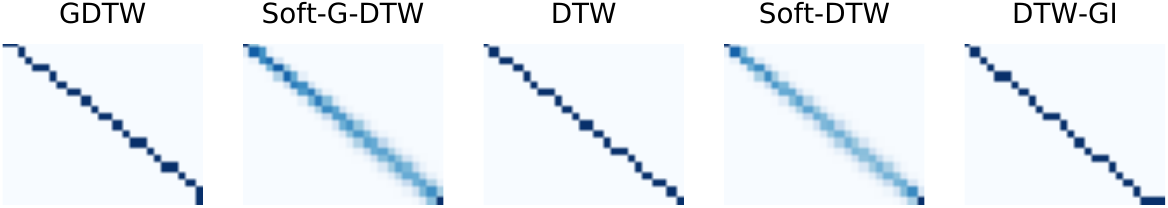}
    \includegraphics{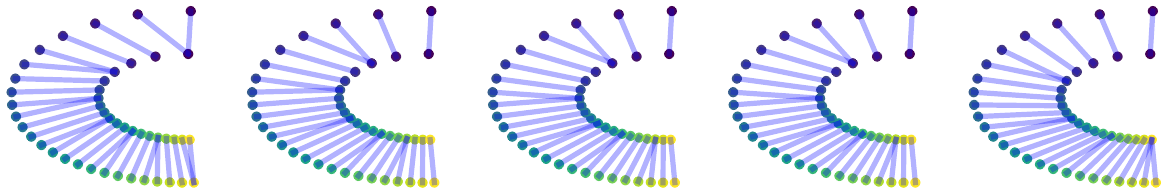}
    \caption{Two time series (top) along with alignment matrices (middle) and alignments with different approaches.  In this example, all methods provide a sensible alignment because the time series are on the same axis of rotation and close in the ground space.  }
    \label{fig:ali_app_1}
\end{figure}

\begin{figure}
    \centering
    \includegraphics{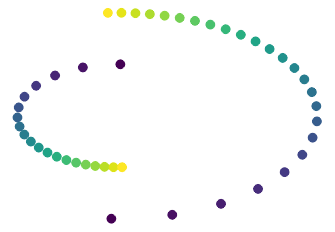}\\[5ex]
    \includegraphics{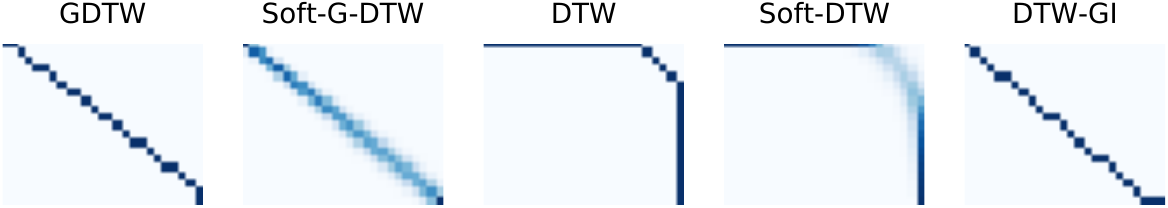}
    \includegraphics{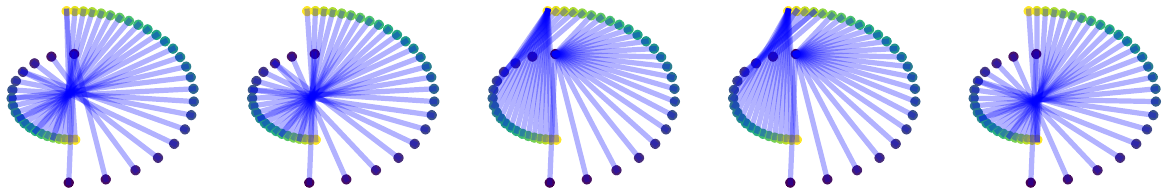}
    \caption{Two time series (top), alignment matrices (middle) and alignments with different approaches.  In this example, the time series are not on the same rotation axis which makes DTW variants fail, whilst GDTW and DTW-GI  (rotation) provide good alignments due to rotational invariance.}
    \label{fig:ali_app_2}
\end{figure}
\clearpage
\begin{figure}
    \centering
    \includegraphics{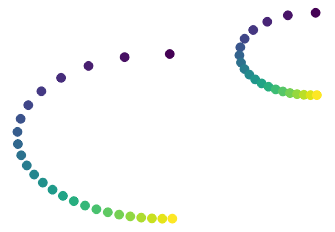}\\[5ex]
    \includegraphics{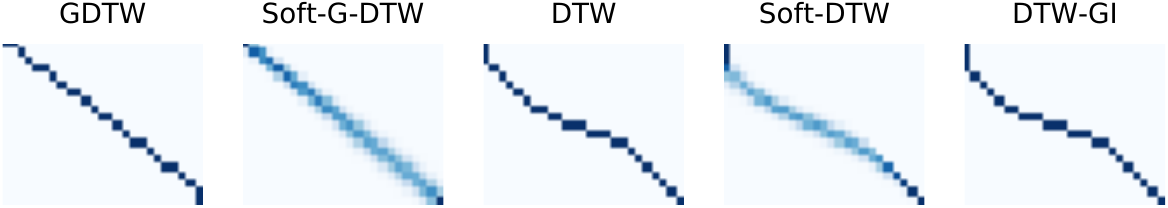}
    \includegraphics{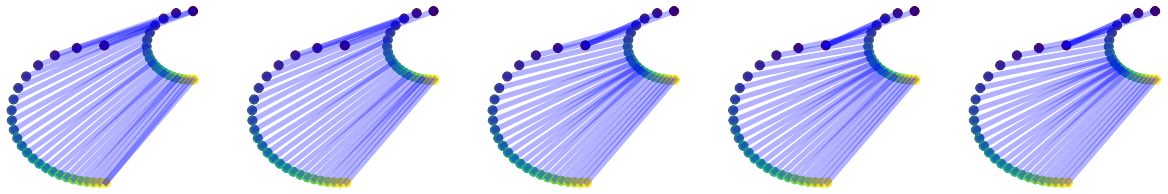}
    \caption{Two time series (top) along with alignment matrices (middle) and alignments with different approaches.  In this example, the time series are translated which makes DTW variants and DTW-GI  (rotation) fail, whilst GDTW is invariant to all isometries, and is thus robust to such transformation.}
    \label{fig:ali_app_3}
\end{figure}

\begin{figure}
    \centering
    \includegraphics{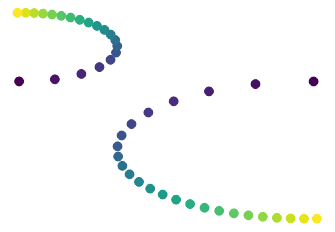}\\[5ex]
    \includegraphics{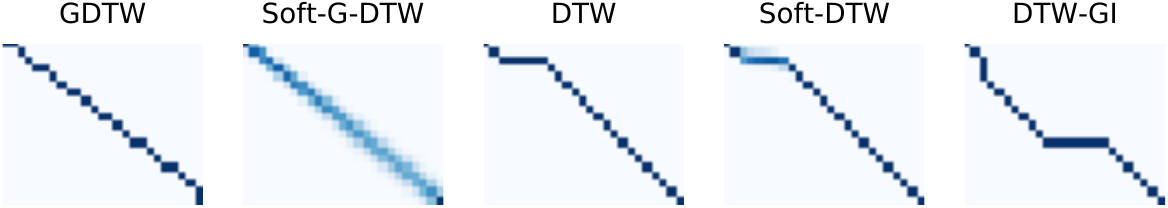}
    \includegraphics{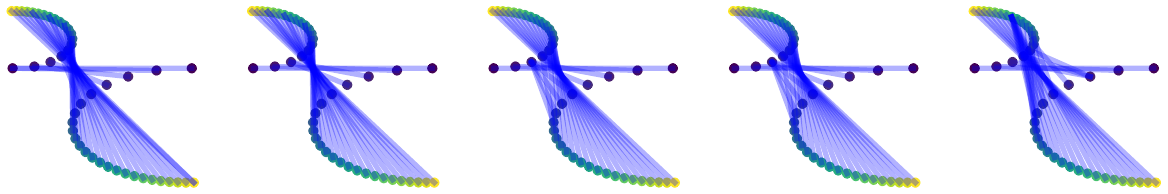}
    \caption{Two time series (top) along with alignment matrices (middle) and alignments with different approaches.  In this example, the time series are rotated and translated which makes DTW variants and DTW-GI  (rotation) fail, whilst GDTW is invariant to all isometries, and is thus robust to such transformations.}
    \label{fig:ali_app_4}
\end{figure}

\end{document}